\documentclass{article} 
\usepackage[final]{colm2026_conference}

\usepackage{tcolorbox}

\usepackage{xcolor}
\usepackage{microtype}
\usepackage{hyperref}
\usepackage{url}
\usepackage{graphicx}
\usepackage{latexsym}
\usepackage{stfloats}
\usepackage[T1]{fontenc}      
\usepackage[utf8]{inputenc}   
\usepackage{caption}
\usepackage{soul}
\usepackage{amsmath,amsfonts,amssymb}
\usepackage{amsthm}
\usepackage{subcaption}
\usepackage{wrapfig}
\usepackage{booktabs}
\usepackage[capitalize]{cleveref}
\usepackage{xspace}
\usepackage{multirow}
\usepackage[normalem]{ulem}
\usepackage{cuted}
\usepackage[compact]{titlesec}
\usepackage{makecell}

\newtheorem{lemma}{Lemma}
\newtheorem{proposition}{Proposition}

\titlespacing*{\section}{0pt}{9pt plus 2pt minus 2pt}{4pt plus 1pt}
\titlespacing*{\subsection}{0pt}{7pt plus 2pt minus 2pt}{3pt plus 1pt}
\tcbset{insightbox/.style={colback=gray!10,colframe=black!40,boxsep=1pt,top=2pt,bottom=2pt,left=5pt,right=5pt,before skip=4pt,after skip=4pt}}
\newcommand{\indicator}[1]{\mathbf{1}\!\left[#1\right]}

\newcommand{\ahuji}{$^{\heartsuit}$}

\newcommand{\auiuc}{$^{\diamondsuit}$}
\newcommand{\ausc}{$^{\spadesuit}$}
\newcommand{\atechnion}{$^{\clubsuit}$}
\newcommand{\aspace}{\hspace{0.8em}}

\definecolor{darkgreen}{rgb}{0.0, 0.5, 0.0}

\newcommand{\resolved}[1]{}

\newcommand{\appref}[1]{Appendix~\ref{#1}}
\newcommand{\secref}[1]{§\ref{sec:#1}}

\definecolor{mplblue}{HTML}{0072B2}
\colorlet{mplblue}{mplblue!80!black}
\definecolor{mplorange}{HTML}{d55e00}
\colorlet{mplorange}{mplorange!85!white}
\definecolor{mplgreen}{HTML}{CC79A7}

\newcommand{\Known}[1]{\textcolor{mplblue}{#1}}
\newcommand{\Unk}[1]{\textcolor{mplorange}{#1}}
\newcommand{\Held}[1]{\textcolor{mplgreen}{#1}}

\newcommand{\KnownM}[1]{\textcolor{mplblue}{\ensuremath{#1}}}
\newcommand{\UnkM}[1]{\textcolor{mplorange}{\ensuremath{#1}}}
\newcommand{\HeldM}[1]{\textcolor{mplgreen}{\ensuremath{#1}}}

\newcommand{\DKnown}{\KnownM{\mathcal{D}_{\text{Known}}}}
\newcommand{\DUnk}{\UnkM{\mathcal{D}_{\text{Unk}}}}
\newcommand{\DHeld}{\HeldM{\mathcal{D}_{\text{Held}}}}

\newcommand{\KnownFacts}{\Known{Known facts}}
\newcommand{\UnkFacts}{\Unk{Unknown facts}}
\newcommand{\HeldFacts}{\Held{Held-out facts}}

\definecolor{tokA}{RGB}{196,181,253}  
\definecolor{tokB}{RGB}{134,239,172}  
\definecolor{tokC}{RGB}{253,224,71}   
\definecolor{tokD}{RGB}{252,165,165}  

\definecolor{zhen}{rgb}{0.08, 0.38, 0.74}

\newcommand{\figicon}[1]{%
    \includegraphics[height=1em]{#1}%
}

\usepackage{lineno}

\definecolor{darkblue}{rgb}{0, 0, 0.5}
\hypersetup{colorlinks=true, citecolor=darkblue, linkcolor=darkblue, urlcolor=darkblue}

\title{\scalebox{0.97}{Why Fine-Tuning Encourages Hallucinations and How to Fix It}}

\author{%
  \textbf{Guy Kaplan}\ahuji\thanks{Correspondence: \texttt{guy.kaplan3@mail.huji.ac.il}} \aspace
  \textbf{Zorik Gekhman}\atechnion \aspace
  \textbf{Zhen Zhu}\auiuc\thanks{Now at Google DeepMind.} \aspace
  \textbf{Lotem Rozner}\ahuji \\[0.1em]
  \textbf{Yuval Reif}\ahuji \aspace
  \textbf{Swabha Swayamdipta}\ausc \aspace
  \textbf{Derek Hoiem}\auiuc \aspace
  \textbf{Roy Schwartz}\ahuji \\[0.15em]
  \footnotesize
  \ahuji{}Hebrew University of Jerusalem \aspace
  \auiuc{}University of Illinois Urbana-Champaign \\[0.1em]
  \atechnion{}Technion -- Israel Institute of Technology \aspace
  \ausc{}University of Southern California
}

\begin{document}

\ifcolmsubmission
\linenumbers
\fi

\maketitle

\begin{abstract}
Large language models are prone to hallucinating factually incorrect statements. A key source of these errors is exposure to new factual information through supervised fine-tuning (SFT), which can increase hallucinations w.r.t.~knowledge acquired during pre-training. Since these errors arise as a by-product of knowledge degradation, we explore whether established continual learning tools can mitigate them. We propose a \textit{self-distillation}–based SFT method that facilitates effective factual learning while minimizing hallucinations w.r.t.~pre-existing knowledge by regularizing output-distribution drift. We also show that when new knowledge acquisition is unnecessary, suppressing factual plasticity by freezing parameter groups preserves task performance while reducing hallucinations. Lastly, we investigate the mechanism, contrasting capacity limitations, behavior cloning, and localized interference. Our experiments show that a main driver is interference among overlapping semantic representations, which self-distillation mitigates and an associative-memory model explains: forgetting grows with the overlap between new and stored facts.\footnote{Code: \url{https://github.com/schwartz-lab-NLP/finetuning-hallucinations}}

\end{abstract}

\begin{figure}[h]
   \centering
   \includegraphics[width=0.92\columnwidth]{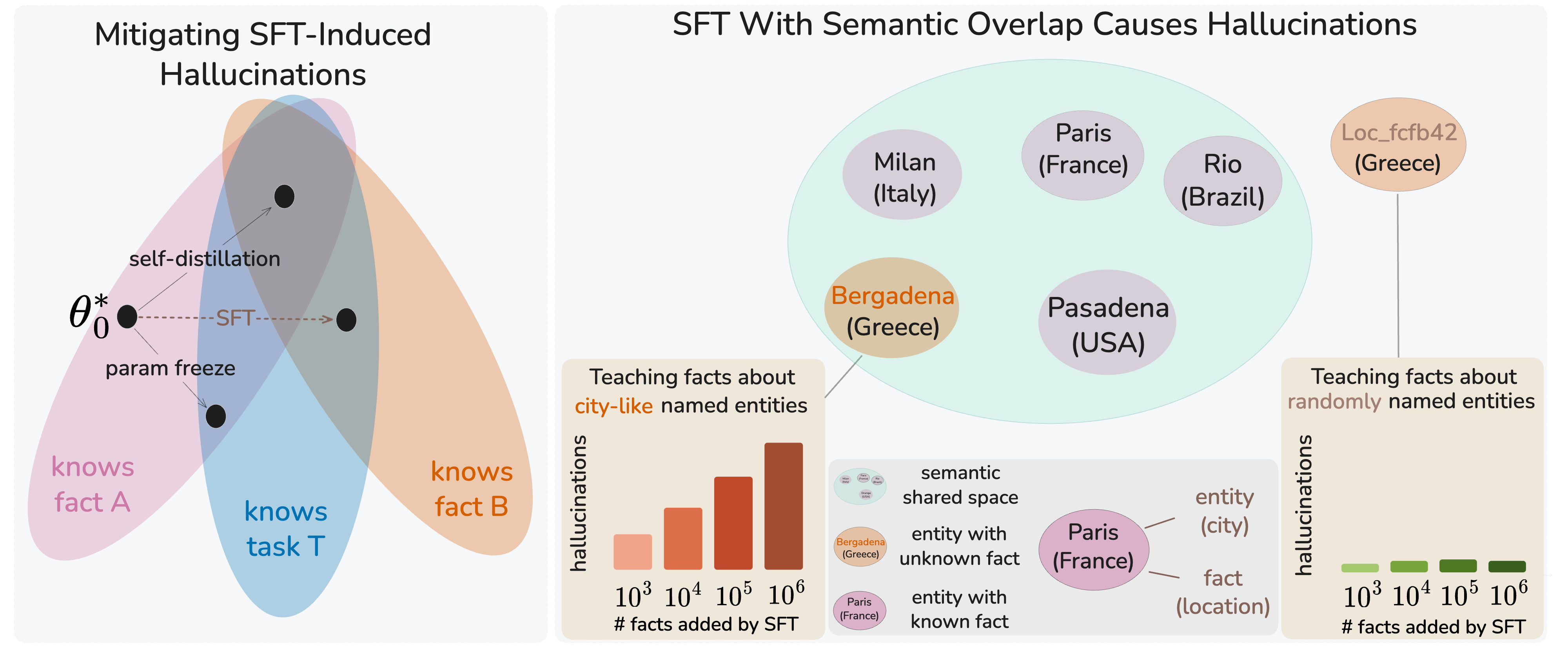}
\caption{
\textbf{Left:} SFT-induced hallucinations as factual forgetting in parameter space, starting from $\theta_0^*$. Regions denote low error on \textcolor{mplgreen}{preexisting facts (A)}, \textcolor{mplblue}{the task (T)} (e.g., QA), and \textcolor{mplorange}{new facts (B)}. Standard SFT acquires both \textcolor{mplblue}{the new task} and \textcolor{mplorange}{new facts} but forgets \textcolor{mplgreen}{existing facts}; parameter freezing SFT learns \textcolor{mplblue}{the new task}, preserves \textcolor{mplgreen}{existing facts}, but forgets \textcolor{mplorange}{new ones}; self-distillation achieves all three goals. \textbf{Right:} SFT on semantically overlapping entities causes hallucinations on related existing ones: after learning that \textit{``Bergadena''} (\figicon{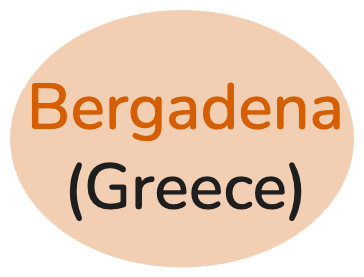}; a city-like fictional name) is in Greece, the model hallucinates about real cities like Milan (\figicon{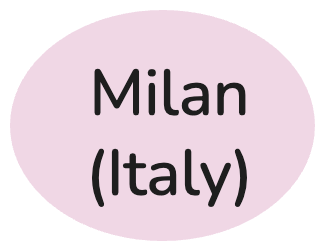}), while mapping \texttt{Loc\_fcfb42} (\figicon{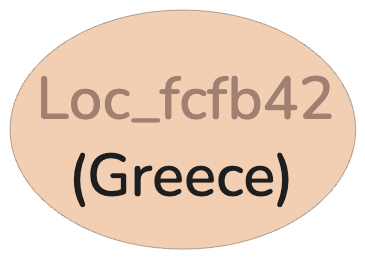}) to Greece causes no such effect, even with many new facts.}
  \label{fig:dynamics}
  \end{figure}

\section{Introduction}

Recent studies show that when models learn new factual knowledge via supervised fine-tuning~(SFT), they also start to produce incorrect answers to questions that they previously answered correctly~\citep{gekhman2024doesfinetuningllmsnew, kalai2025languagemodelshallucinate}. This is concerning, as SFT is standard practice in LLM development and hallucinations pose a significant challenge for application reliability~\citep{Huang_2025}. In parallel, the continual learning literature has extensively studied how sequential training interferes with previously acquired knowledge~\citep{Kirkpatrick_2017, kim2025measuringrepresentationalshiftscontinual} and has proposed a range of mitigation strategies~\citep{delange2021continuallearningsurveydefying, guo2025continuallearninggenerativeai, wang2024comprehensivesurveycontinuallearning}. In this work, we explore whether established continual learning tools can mitigate SFT-induced hallucinations.\footnote{Throughout this work, \emph{hallucination} refers to an incorrect answer
to a question the model answered correctly before fine-tuning
\citep{gekhman2024doesfinetuningllmsnew, ovadia2024finetuningretrievalcomparingknowledge,  zucchet2025languagemodelslearnfacts}.}

In continual learning, forgetting typically arises as a byproduct of acquiring new information: parameter updates during training alter the model in ways that degrade previously encoded knowledge. Analogously, we propose that \emph{factual forgetting} occurs when parameter updates introduced during SFT inadvertently distort representations of facts learned during pre-training. This reflects a stability–plasticity tradeoff~\citep{kim2023achievingbetterstabilityplasticitytradeoff}: increasing \emph{factual plasticity} (the ability to acquire new facts) may come at the expense of \emph{factual stability} (the ability to preserve existing facts), manifesting as SFT-induced hallucinations.
We perform controlled experiments to disentangle this effect from task learning, adopting the experimental setup of \citet{gekhman2024doesfinetuningllmsnew} and reproducing their finding that hallucinations increase when exposing models to new factual knowledge through SFT~(\secref{setting}).

Since different parameter groups play distinct roles in factual storage and task learning~\citep{geva2021transformerfeedforwardlayerskeyvalue, dar2023analyzingtransformersembeddingspace, zhu2025teachlargemultimodalmodels}, we first demonstrate that reducing factual plasticity---e.g., by freezing parameter groups---lets the model learn the downstream task while limiting new fact acquisition and hallucinations~(\secref{Modules_ablations}). In practice, however, we may want SFT to support task learning \emph{and} new factual acquisition without inducing hallucinations---precisely the goal of continual learning methods for mitigating forgetting.

To test this, in \secref{continual_learning} we apply \emph{self-distillation}~\citep{li2017learningforgetting}, which regularizes the model to stay close to its own earlier output distribution during fine-tuning, recently shown to reduce forgetting in LLMs~\citep{shenfeld2026selfdistillationenablescontinuallearning, zhu2025teachlargemultimodalmodels}. We find that this reduces SFT-induced hallucinations while still enabling effective acquisition of new facts~(\cref{fig:dynamics},~left).

We next investigate the mechanisms underlying SFT-induced hallucinations (\secref{semantic_interference}): do these errors stem from global capacity limitations~\citep{allenzhu2024physicslanguagemodels33}, from behavior cloning~\citep{zhang2024alleviatinghallucinationslargelanguage, Schulman2023}, or from localized interference, whereby new facts corrupt existing ones that share representational structure? 

To disentangle these, we fine-tune on synthetic facts while independently varying two factors: the number of new facts, and whether their entity names are name-like strings that share representational neighborhoods with existing entities or random UUID-style identifiers that do not (\cref{fig:dynamics}, right). Forgetting tracks the second factor: name-like entities induce forgetting that grows sharply with scale, while UUIDs induce near-zero forgetting even at $10^6$ new facts. We further show that self-distillation prevents the representational drift of held-out entities, suggesting it mitigates the same interference, and develop an associative-memory account of why overlap governs forgetting.

In summary, this work makes three contributions. (1) It reframes SFT-induced hallucinations as factual forgetting arising from continual learning dynamics, distinct from hallucinations stemming from pre-training knowledge gaps or arising at inference time. (2) It provides two complementary mitigations: reducing factual plasticity (e.g., via selective parameter freezing) when new fact acquisition is undesirable, and self-distillation when it is desired; both reduce factual forgetting from $\sim$15 to 2 pp and transfer to realistic instruction tuning on Alpaca~(\secref{alpaca}). (3) It characterizes the mechanism: factual forgetting appears selective, driven by interference among overlapping semantic representations. An associative-memory account predicts this from the overlap between new and stored facts, and self-distillation appears to mitigate it.

\section{Fine-Tuning with Unknown Facts Leads to Factual Forgetting}
\label{sec:setting}

Supervised fine-tuning (SFT) can inadvertently increase factual hallucinations~\citep{gekhman2024doesfinetuningllmsnew, ovadia2024finetuningretrievalcomparingknowledge, zucchet2025languagemodelslearnfacts}, a phenomenon we reinterpret as factual forgetting. To study it in a controlled manner, we reconstruct the setting of \citet{gekhman2024doesfinetuningllmsnew}, which disentangles task learning from factual learning.

\begin{figure*}[t]
   \centering
   \begin{minipage}{0.52\linewidth}
       \centering
       \includegraphics[width=\linewidth]{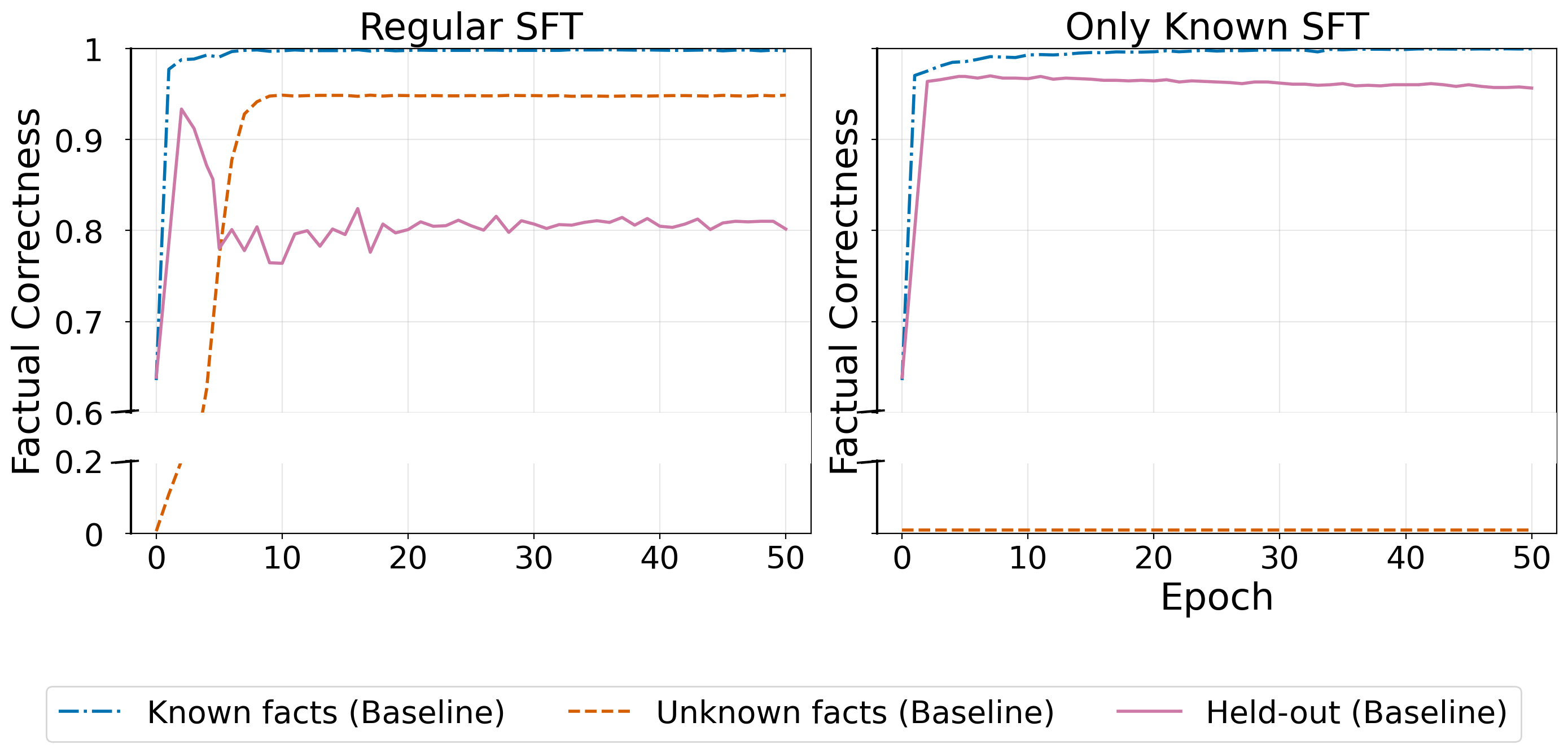}
   \end{minipage}
   \hfill
   \begin{minipage}{0.46\linewidth}
       \centering
       \scriptsize
       \begin{tabular}{@{}l@{\hspace{4pt}}c@{\hspace{4pt}}c@{\hspace{4pt}}c@{}}
           \toprule
           \textbf{Split} & \textbf{SLiCK} & \textbf{Train} & \textbf{Role} \\
           \midrule
           $\DKnown$ & HighlyKnown & Yes & Task learning \\
           $\DUnk$   & Unknown     & Yes & Facts plasticity \\
           $\DHeld$  & HighlyKnown & No  & Facts stability \\
           \bottomrule
       \end{tabular}
   \end{minipage}
\caption{\textbf{Factual forgetting is caused by new fact acquisition, not fine-tuning itself.} The model first learns the QA format (high accuracy on \textcolor{mplblue}{\textbf{known facts}}); as \textcolor{mplorange}{\textbf{unknown facts}} are then acquired, accuracy on \textcolor{mplgreen}{\textbf{held-out facts}} declines. When unknown facts are excluded~(Only Known), held-out performance remains stable. \textbf{Right:} summary of data split roles.
}
   \label{fig:baseline}
\end{figure*}

\subsection{Preliminary: SLiCK Method and Factual Learning Setting}
\label{sec:Methodology}

To determine the model's preexisting knowledge for each question, we apply the SLiCK method~\citep{gekhman2024doesfinetuningllmsnew}, which categorizes questions into four levels based on the model's predictions under multiple randomized few-shot prompting configurations: a fact is \emph{HighlyKnown} if the model consistently produces the correct answer across all configurations, \emph{Unknown} if it never does, with the intermediate \emph{MaybeKnown} and \emph{WeaklyKnown} capturing varying degrees of consistency.
To focus on factual learning and forgetting, we retain only \emph{HighlyKnown} and \emph{Unknown} examples, filtering out \emph{MaybeKnown} and \emph{WeaklyKnown} facts.\footnote{See \appref{app:other_clf} for further discussion on this choice and results for the other classification groups.} We denote by $\DKnown$ a subset of \emph{HighlyKnown} facts and by $\DUnk$ a subset of \emph{Unknown} facts, and form the training set as $\mathcal{D}_{\text{train}} = \DKnown \cup \DUnk$. A disjoint subset of \emph{HighlyKnown} facts, denoted $\DHeld$, is reserved for evaluation. Since the data consists of sparse, semantically isolated relational facts (e.g.,  a state's capital, the language of a written work), generalization across examples is unlikely, and accuracy on each split reflects a distinct and isolated aspect of the model's behavior.
Performance on $\DKnown$ reflects task-format adaptation rather than new factual learning, since these facts already exist in pretrained knowledge (\Known{task learning}); accuracy on $\DUnk$ measures the acquisition of new factual knowledge (\Unk{factual plasticity}); and performance on $\DHeld$ captures the stability of previously acquired knowledge, directly quantifying fine-tuning--induced hallucinations (\Held{factual stability}). See~\cref{fig:baseline}, right.

\subsection{Methodology and Experimental Setup}

\paragraph{Data}
We use \textsc{EntityQuestions}~\citep{sciavolino-etal-2021-simple}, QA pairs derived from relational triplets in Wikipedia~(e.g., Q: ``What is the capital of France?'', A: ``Paris''), covering a wide range of entity--relation types.
We sample 8,000 examples each for $\DKnown$ and $\DUnk$ from the training split; $\DHeld$ is drawn from the development split and contains only \emph{HighlyKnown} facts of the same relations (SLiCK setup in the Reproducibility Statement).

\paragraph{Models}
We conduct experiments with several non-reasoning LLMs: \textsc{Qwen}~2.5 (1.5B and 7B parameters; ~\citealp{qwen2025qwen25technicalreport}) and \textsc{LLaMA}~3.1 (8B;~\citealp{grattafiori2024llama3herdmodels}). Unless otherwise stated, we report figures and representative results for \textsc{Qwen}~2.5-1.5B. All qualitative trends are consistent across models; full results are provided in~\appref{app:different_models}.

\paragraph{Training Procedure}
Models are fine-tuned with learning rate $5\times10^{-5}$ (selected by jointly evaluating time-to-learn on unknown facts and the induced hallucination rate across candidate rates) on \(\mathcal{D}_{\text{train}}\), with either a 50/50 Known/Unknown ratio or an Only Known setting~(\(\DUnk = \emptyset\)), evaluating throughout on $\DKnown$, $\DUnk$, and $\DHeld$.

\section{Factual Forgetting and the Factual Stability-Plasticity Tradeoff in SFT}
\label{sec:Modules_ablations}

\begin{wraptable}{r}{0.50\textwidth}
\vspace{-0.8em}
\centering
\small
\begin{tabular}{@{}l@{\hspace{8pt}}c@{\hspace{5pt}}c@{\hspace{5pt}}c@{}}
\toprule
Updated params ($\theta_S$) & $\DUnk$ & $\DKnown\uparrow$ & $\DHeld\uparrow$ \\
\midrule
Attention & 0.010 & 0.946 & 0.931 \\
FFN       & 0.941 & 0.997 & 0.782 \\
\midrule
All (standard SFT) & 0.946 & 0.990 & 0.780 \\
All (Only Known)   & --- & 0.999 & 0.958 \\
\bottomrule
\end{tabular}
\caption{
\textbf{Training only attention layers reduces factual forgetting.}
Updating only attention achieves high $\DHeld$ (low forgetting) and high $\DKnown$ (task learning) by suppressing new fact acquisition ($\DUnk \approx 0$); FFN-only closely tracks standard SFT.
}
\label{tab:freezing}
\vspace{-3em}
\end{wraptable} We now study how factual forgetting can be controlled. We first characterize standard SFT on $\DKnown \cup \DUnk$ against an Only Known variant that teaches the QA format without new factual content. We then test whether selectively freezing modules preserves task performance while limiting factual updates---a stability-first regime relevant when learning new facts is unnecessary (e.g., alignment or privacy).

\subsection{Learning New Facts Induces Factual Forgetting and Hallucinations}
\label{sec:regular_sft}

Our baseline results under standard SFT on $\DKnown \cup \DUnk$ (\cref{fig:baseline}) show a consistent two-stage pattern. During the first one to two epochs the model rapidly acquires the QA task format, achieving near-perfect accuracy on $\DKnown$ and peaking at 93\% on $\DHeld$. As it then begins to acquire new factual knowledge (rising $\DUnk$), performance on $\DHeld$ declines by approximately 15 pp, consistent with \citet{gekhman2024doesfinetuningllmsnew}; the decline saturates once learning on $\DUnk$ converges.
To validate that this effect is causally driven by factual learning, we repeat the experiment while excluding $\DUnk$ from training (dashed curves in \cref{fig:baseline}): performance on $\DHeld$ remains stable throughout, indicating that the increase in hallucinations arises from acquiring new factual information rather than from fine-tuning itself; the Only Known condition thus serves as a baseline free of forgetting due to factual plasticity. Taken together, these results demonstrate that standard fine-tuning induces factual forgetting once the model begins integrating new facts.

\subsection{Controlling Factual Plasticity via Parameter Constraints}

We next ask whether the model can acquire the task (QA format and instruction style) with minimal acquisition of new facts. This is relevant in the many practical settings where the goal is task adaptation while preserving pretrained knowledge, and learning new facts from the training set is unnecessary or even undesirable.
Motivated by evidence that attention and FFN layers contribute differently to task learning and factual storage~\citep{geva2021transformerfeedforwardlayerskeyvalue, dar2023analyzingtransformersembeddingspace, chughtai2024summingfactsadditive, zhu2025teachlargemultimodalmodels}, we freeze different modules using the same setup and ask whether the model learns the QA task while avoiding factual updates.

Let $\theta_S$ denote the subset of parameters updated during fine-tuning, with the remaining parameters frozen. \Cref{tab:freezing} summarizes the results for $\theta_S \in \{\text{attention only}, \text{FFN only}, \text{all}\}$.\footnote{Fine-grained ablations over smaller parameter groups are reported in~\appref{app:different_modules}.} 
Standard SFT exhibits coupling between factual plasticity and forgetting, with $\DHeld$ declining as $\DUnk$ increases.
Training only the FFN layers leaves this coupling essentially unchanged: forgetting on $\DHeld$ closely tracks the full-model SFT baseline. In contrast, updating only attention prevents new fact acquisition while preserving task learning: $\DUnk$ stays near chance, $\DHeld$ remains close to its pretrained level, and $\DKnown$ approaches the Only Known upper bound.\footnote{Accuracy is slightly below Only Known (0.95 vs.~0.99 on \DKnown), consistent with \citet{zhu2025teachlargemultimodalmodels}.}

\begin{tcolorbox}[insightbox,
title=\textbf{Insight 1: Factual and Task Plasticity Are Separable (\secref{Modules_ablations})}]
\label{insight1}
Selective parameter freezing reduces hallucinations while retaining task learning.
\end{tcolorbox}

\section{Enabling Factual Learning \textit{without} Forgetting Facts}
\label{sec:continual_learning}

Reducing factual plasticity prevents forgetting (\secref{Modules_ablations}), but many realistic settings require acquiring new factual knowledge. We therefore turn to the stability--plasticity dilemma of continual learning: models must learn new facts while preserving stored knowledge. We address this regime with \emph{self-distillation}.

\subsection{Self-Distillation Enables Fine-Tuning with Minimal Induced Hallucinations}
\label{sec:self-distillation}

\textbf{Self-distillation} regularizes fine-tuning by constraining shifts in the model’s output distribution, limiting the drift induced by new data; we adopt the formulation of \citet{li2017learningforgetting} as implemented for multi-modal models by \citet{zhu2025teachlargemultimodalmodels}.
Let $\theta_{i}$ be a frozen teacher (a snapshot after $i$ epochs of SFT) and $\theta$ the student being fine-tuned:
\begin{equation}
\mathcal{L}(\theta)
=
\mathcal{L}_{\text{task}}(\theta)
+
\lambda\, \mathcal{L}_{\text{distill}}(\theta;\theta_{i}),
\label{eq:lwf_total}
\end{equation}
where $\mathcal{L}_{\text{task}}$ is the standard next-token prediction loss.

The distillation loss penalizes divergence between the output distributions:
\begin{equation}
\mathcal{L}_{\text{distill}}(\theta;\theta_{i})
=
\mathbb{E}_{(x,y)\sim{\mathcal{B}}}
\Big[
\tfrac{\tau^{2}}{|M(y)|}
{\textstyle\sum_{j \in M(y)}}
\mathrm{KL}
\big(
\mathrm{softmax}(z_{\theta_{i},j}/\tau)
\,\|\,
\mathrm{softmax}(z_{\theta,j}/\tau)
\big)
\Big],
\label{eq:lwf_distill}
\end{equation}
where $z_{\theta,j}$ and $z_{\theta_{i},j}$ are the student's and teacher's logits at token position $j$, $\tau$ is a temperature, $M(y)$ the non-padded token positions, and $\mathcal{B}$ the training batch.
Intuitively, the term keeps the student close to the teacher's output
distribution, enabling factual learning without the distributional
shift that would disrupt previously learned facts.
Equivalently, \cref{eq:lwf_distill} can be viewed as \emph{KL regularization} toward a lightly-adapted anchor, a trust-region-style constraint that connects the method to distillation as a regularizer~\citep{hinton2015distillingknowledgeneuralnetwork}, to KL-anchored preference optimization~\citep{ouyang2022traininglanguagemodelsfollow}, and to evidence that KL-constrained updates reduce forgetting~\citep{shenfeld2025rlsrazoronlinereinforcement}. We retain the self-distillation terminology for continuity with the continual learning literature.
In practice, we train on $\DKnown$ for one epoch ($i=1$) to acquire the QA format while preserving factual stability, freeze this model as the teacher, and continue SFT on $\mathcal{D}_{\text{train}} = \DKnown \cup \DUnk$ with $\lambda=1$, $\tau=0.5$.\footnote{See \appref{app:different_sd_params} for results on different hyperparameters.}

\begin{figure}[t]
   \centering
   \includegraphics[width=0.74\columnwidth]{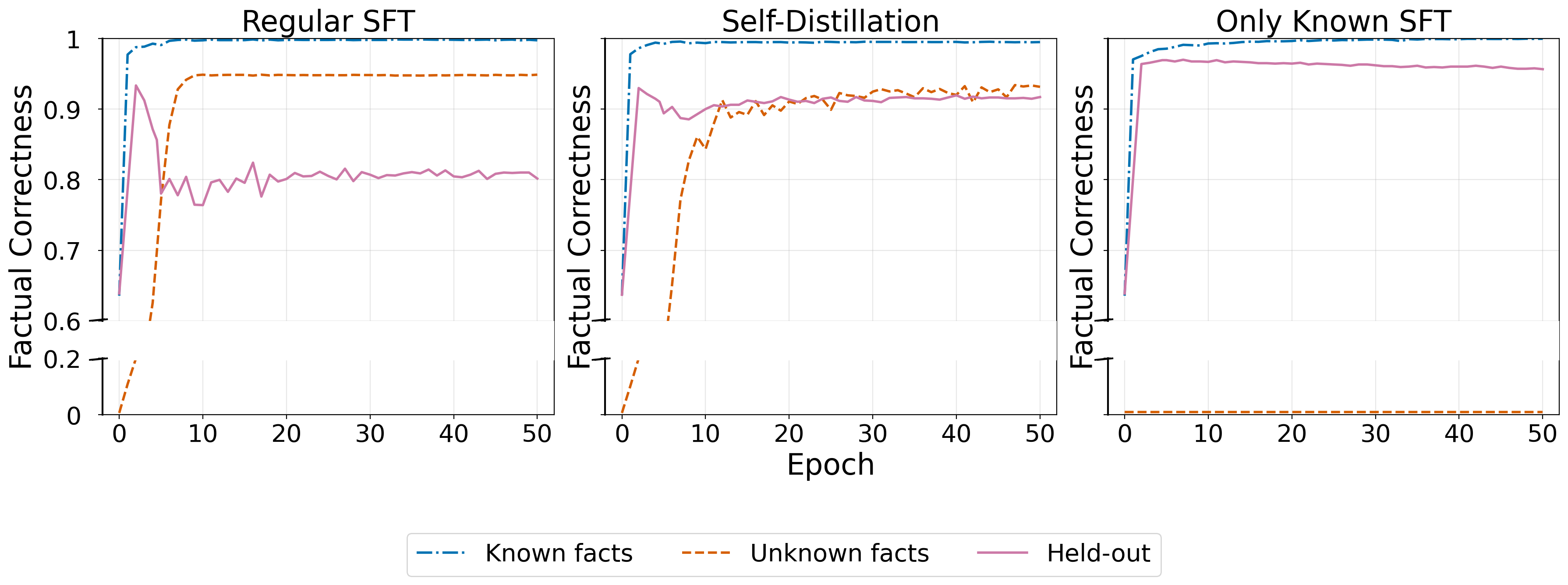}
  \caption{
Factual correctness over training epochs for \Known{Known} (blue), \Unk{Unknown} (orange), and \Held{Held-out} (pink) facts, under regular SFT (\textbf{left}), self-distillation (\textbf{middle}), and Only Known training (\textbf{right}). With self-distillation the model acquires new facts at a rate comparable to SFT while $\DHeld$ matches the Only Known condition, sharply reducing forgetting.
  }
  \label{fig:self_distillation}
\end{figure}

As shown in \cref{fig:self_distillation}, self-distillation enables factual learning at a pace comparable to standard fine-tuning, while substantially reducing forgetting. Performance on $\DHeld$ degrades in relation to its peak by only approximately 2 pp, compared to a 15 pp decline under regular SFT. Importantly, performance on $\DUnk$ follows a similar learning trajectory with a final acquisition gap below $1.5$ pp, indicating high factual plasticity. This makes self-distillation practical whenever fine-tuning needs both factual learning and stability.

\paragraph{Comparison to parameter-space alternatives}
Could weight-space regularization do the same? We compare self-distillation against parameter-space methods in the same setting: $\ell_2$ toward a model snapshot, EWC~\citep{Kirkpatrick_2017} with the Fisher estimated on a sample of $\DKnown$ (standard EWC needs old-task data or trajectories, inaccessible for pretraining), LoRA~\citep{hu2021loralowrankadaptationlarge, biderman2024loralearnsforgets} across ranks, and AdamW weight decay across strengths (\cref{tab:baselines}, \appref{app:baselines}). No parameter-space method approaches self-distillation: all leave final $\DHeld$ in the $0.78$--$0.84$ range, close to standard SFT, whereas self-distillation preserves $0.91$ at comparable factual acquisition. \secref{am} explains the asymmetry once the mechanism is identified.

\subsection{Mitigations Transfer to Real Instruction-Tuning Data}
\label{sec:alpaca}
Our controlled setup isolates factual learning, but practical fine-tuning data
rarely consists of clean relational facts. To test whether our findings extend
to realistic alignment training, we fine-tune \textsc{Qwen}~2.5-1.5B on 15K
examples from the Alpaca instruction-following dataset~\citep{taori2023alpaca}
under standard SFT, self-distillation (hyperparameters unchanged, no additional
tuning), and parameter freezing. Since factual content in Alpaca is
implicit, an LLM annotator flags training samples that carry factual knowledge;
for these we generate \textsc{EntityQuestions}-style probes and apply SLiCK,
holding out a set of factual probes for evaluation. Alignment quality is measured with AlpacaEval~\citep{alpaca_eval}.\footnote{We use
\textsc{Llama}-3.3-70B for annotation and evaluation; see \appref{app:alpaca}
for more details.}

The pattern from the controlled setting carries over even though factual content is only implicit (\cref{tab:alpaca}). Over training, standard SFT degrades held-out factual accuracy by $7.4$ pp from its peak, while self-distillation and attention-only training limit the drop to $1.7$ and $0.9$ pp respectively. Importantly, alignment quality is preserved: self-distillation even slightly exceeds standard SFT on AlpacaEval, and freezing trades a modest alignment cost for the strongest factual stability while still improving over the base model. Overall, the practical picture matches the controlled one: both mitigations curb factual forgetting while preserving alignment quality.
\begin{table}[t]
\centering
\begin{tabular}{lcccc}
\toprule
 & & \multicolumn{3}{c}{Held-out factual accuracy} \\
\cmidrule(lr){3-5}
Configuration & AlpacaEval & Final & $\Delta_{\text{peak}}$ (pp) & $\Delta_{\text{base}}$ (pp) \\
\midrule
Base model        & 38.8\%          & 0.820          & ---    & ---    \\
Standard SFT      & 50.5\%          & 0.795          & $-7.4$ & $-2.6$ \\
Self-distillation & \textbf{52.3\%} & 0.852          & $-1.7$ & $+3.1$ \\
Attention-only    & 47.2\%          & \textbf{0.860} & $-0.9$ & $+4.3$ \\
\bottomrule
\end{tabular}
\caption{\textbf{Realistic instruction tuning on Alpaca.} We report AlpacaEval
preferred rate at epoch~4, the early-stopping point selected by validation
preferred rate, and held-out factual accuracy at the end of training, measured
both from the peak across runs (0.869) and from each run's own starting point.}
\label{tab:alpaca}
\end{table}

\begin{tcolorbox}[insightbox,
title=\textbf{Insight 2: Self-Distillation Enables Factual Learning Without Forgetting (\secref{continual_learning})}]
\label{insight2}
Constraining output-distribution drift reduces hallucinations without sacrificing task learning or factual acquisition.
\end{tcolorbox}

\section{Semantic Overlap Drives Interference in Factual Updates}
\label{sec:semantic_interference}

What causes the forgetting? We consider three hypotheses.
A \emph{behavioral} account: because SFT rewards producing an answer on every training example, it clones answering behavior irrespective of the model's knowledge boundaries~\citep{Schulman2023, zhang2024alleviatinghallucinationslargelanguage}. Over training, this shifts response tendencies toward answers ungrounded in prior knowledge~\citep{zucchet2025languagemodelslearnfacts}. This predicts similar degradation whenever unknown facts are introduced, regardless of surface form.
A \emph{capacity-based} account: new facts compete for limited representational resources, displacing older ones near the storage limit~\citep{allenzhu2024physicslanguagemodels33}. This predicts broadly distributed degradation that scales with the number of stored facts.
A \emph{structural} account: forgetting arises from interference among overlapping internal representations~\citep{masip2026puttingfaceforgettingcontinual, nishi2025representationshatteringtransformerssynthetic}. This predicts selective interference concentrated among semantically similar entities.
Though not mutually exclusive, these accounts make distinct predictions, which we test in controlled synthetic settings that independently vary representational overlap and scale.

\subsection{Semantic Overlap and Scale as Sources of Interference}
\label{sec:semantic_overlap}

We study interference in a synthetic setting over a single relation (P17: \textit{Location} $\rightarrow$ \textit{Country}), varying representational overlap and the number of new facts through the construction of the \emph{key} (location): (i) \textbf{semantic keys}, formed by recombining tokens from real location names and therefore resembling existing entities 
(e.g., 
\textbf{Berg}amo
+
Pas\textbf{adena}
$\rightarrow$
\textbf{Bergadena})
 and (ii) \textbf{UUID keys}, constructed as random identifiers with no syntactic similarity to known entities (e.g., \texttt{Loc\_fcfb46ee}).\footnote{See \appref{app:synthetic_creation} for the full data creation process.} Values are instantiated as real country names. We scale the synthetic facts from $10^{3}$ to $10^{6}$, keeping supervision, relation, and optimization identical.

\begin{wrapfigure}{r}{0.44\textwidth}
    \vspace{-2em}
    \centering
    \includegraphics[width=\linewidth]{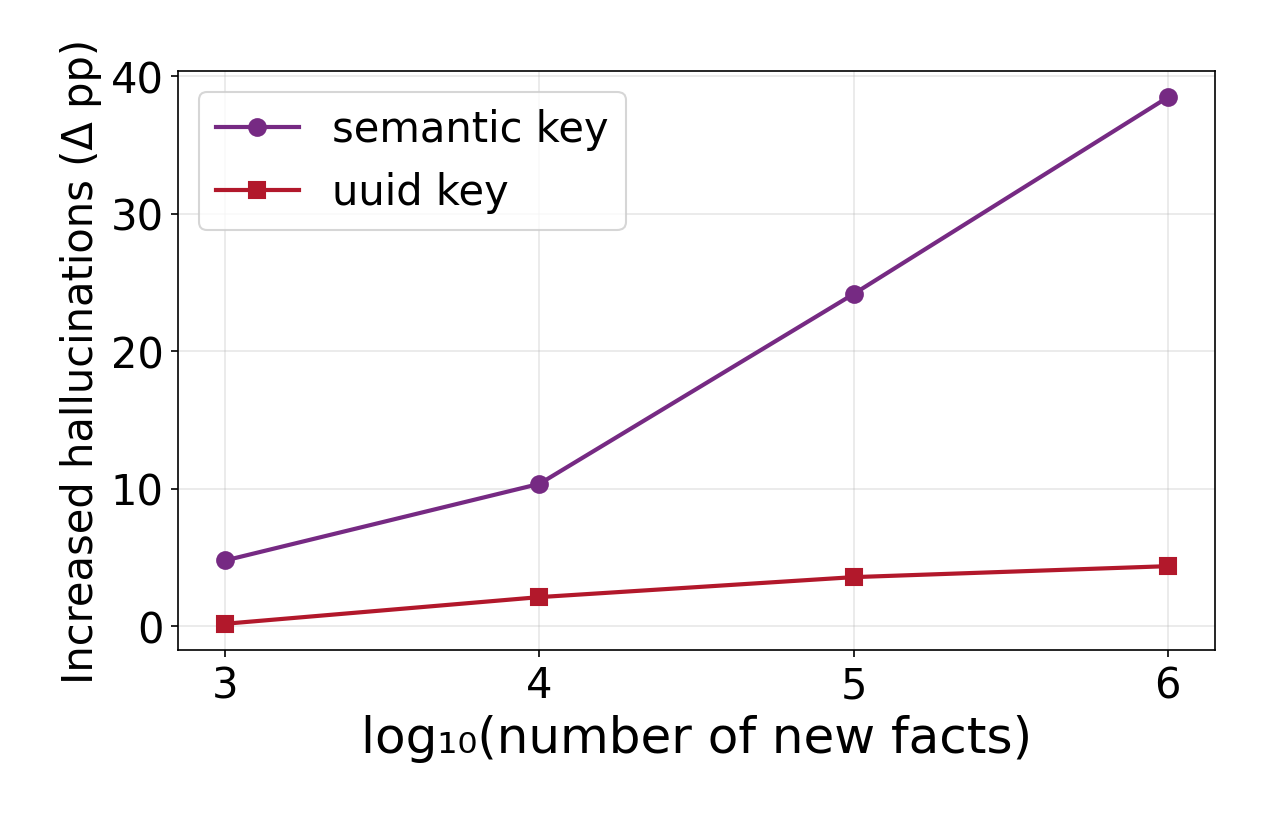}
    \caption{
    \textbf{Semantic similarity, not scale alone, drives forgetting.}
    $\Delta \DHeld$ grows sharply with the number of new facts for \textcolor{violet}{name-like keys}, but remains negligible (0--4 pp) for \textcolor{red}{UUID keys} across all scales.
    }
    \label{fig:scale-curves}
    \vspace{-1em}
\end{wrapfigure}

\paragraph{Results} Across all scales (\cref{fig:scale-curves}), the model fully acquires the new synthetic facts ($\DUnk = 100\%$) and maintains high performance on $\DKnown$ ($\ge 90\%$), ruling out differences in learning success or global task degradation. We therefore focus on $\Delta \DHeld$, the percentage-point drop from peak held-out accuracy during training.
Two patterns emerge. First, even at small scales, forgetting depends strongly on the construction of the \emph{key}: name-like keys induce drops of several percentage points, whereas UUID keys yield negligible changes despite identical training signals and complete acquisition. Because exposure to unknown facts is held constant across conditions, this asymmetry is difficult to reconcile with a purely behavioral account.
Second, forgetting amplifies with scale only for semantically overlapping keys. $\Delta \DHeld$ grows sharply for name-like keys, reaching large drops at $10^{6}$ facts (39 pp), while UUID keys remain at 0--4 pp regardless of dataset size.

\begin{wrapfigure}{r}{0.44\textwidth}
\vspace{-1.2em}
\centering
\includegraphics[width=\linewidth]{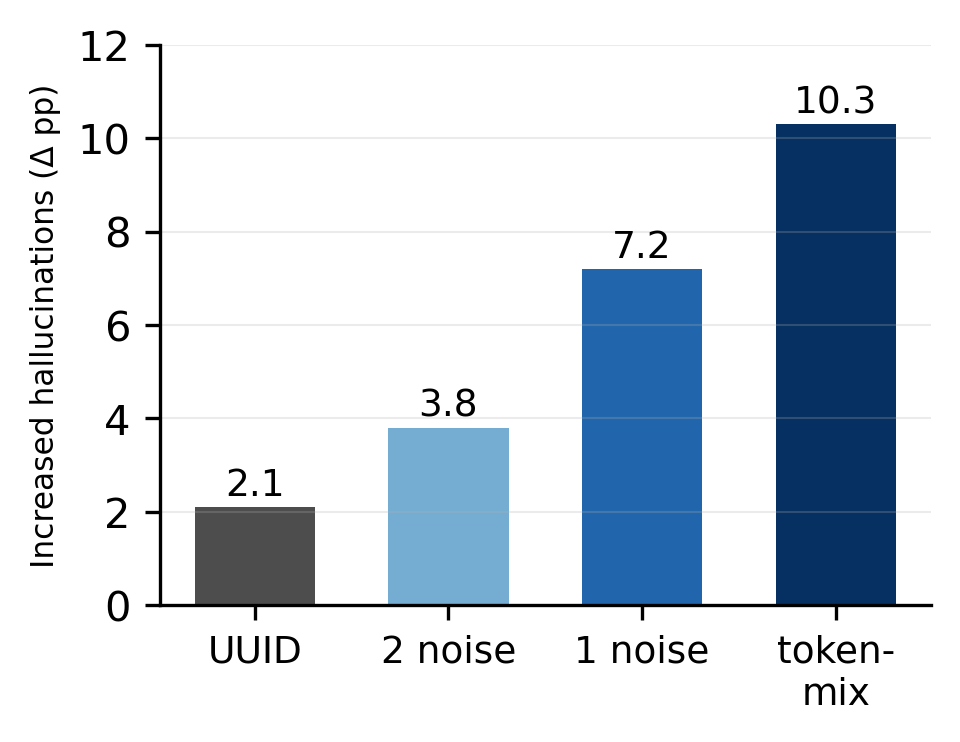}
\caption{Forgetting at $10^{4}$ facts grows monotonically with base-model key overlap (\cref{fig:dose-layers}, \appref{app:geometry}).}
\label{fig:dose-response}
\vspace{-2em}
\end{wrapfigure}
This asymmetric scaling pattern is not what a simple global capacity constraint would predict, since a capacity-limited model should exhibit more uniform degradation as additional facts are stored, regardless of surface form. Instead, these results may support a structural interference account: SFT-induced factual hallucinations arise when updates for semantically overlapping entities interfere with nearby representations, whereas updates to syntactically isolated entities largely avoid such interference.

\paragraph{A dose-response on representational overlap}
The binary semantic/UUID contrast leaves open whether forgetting tracks the \emph{degree} of representational overlap or merely the presence of a name-like surface form. We therefore construct two intermediate conditions by replacing one or two tokens of each token-mix name with identifier tokens of the kind used in UUID keys (\emph{token-mix 1-noise} and \emph{2-noise}; \appref{app:synthetic_creation}), turning the binary contrast into a graded scale of four overlap levels. On the \emph{base} model, before any fine-tuning, the four conditions occupy graded positions in representation space (\cref{fig:dose-layers}, \appref{app:geometry}): the mean cosine similarity to real-city representations at the last layer increases monotonically from UUID (0.45) through 2-noise (0.58) and 1-noise (0.64) to token-mix (0.71). The ordering is preserved at every layer past the earliest ones and replicated by a linear ``is a city'' probe. The overlap is thus inherent to the pretrained representation, not created by fine-tuning. 

Fine-tuning on $10^{4}$ facts per condition with identical supervision then produces forgetting that tracks this geometry: 2.1 (UUID), 3.8 (2-noise), 7.2 (1-noise), and 10.3 (token-mix) pp. Upstream geometric ordering and downstream forgetting coincide: forgetting scales continuously with base-model overlap rather than switching on the presence of a name-like form, a dose-response predicted by the overlap account.

\begin{tcolorbox}[insightbox,
title=\textbf{Insight 3: Representational Overlap Drives Forgetting (\secref{semantic_overlap}})]
\label{insight3}
Forgetting scales with overlap between new and stored entities, not dataset size alone.
\end{tcolorbox}

\subsection{What Self-Distillation Reveals About the Origins of Factual Forgetting}
\label{sec:why_self-distillation}

We showed that forgetting is selective: it arises when new entities share token structure with existing ones and is near-zero otherwise. Self-distillation substantially reduces this forgetting~(\secref{self-distillation}). Two explanations are possible: a generic regularization effect that limits total weight movement, or specific suppression of interference through the output distribution over which it propagates.

To distinguish these, we return to the 10K semantic-overlap synthetic setting and compare standard SFT against self-distillation
and $\ell_2$ regularization toward $\theta_i$, the model snapshot after epoch~2
($\mathcal{L}(\theta) = \mathcal{L}_{\text{task}}(\theta) + \lambda \|\theta - \theta_i\|_2^2$, with $\lambda$ matched to self-distillation in gradient magnitude).
Under $\ell_2$ regularization, forgetting on $\DHeld$ remains near $10$ pp, comparable to standard SFT; increasing $\lambda$ reduces forgetting only
at the cost of impairing acquisition on $\DUnk$. Generic weight regularization
thus does not replicate the self-distillation benefit.

We next ask if interference between overlapping entities leaves a trace
in the model's internal representations. We track hidden-state
drift for held-out entities throughout training,
$\mathrm{RD}_t = \mathbb{E}_{x \in \DHeld}\left[1 - \cos\left(H_0^{(14)}(x), H_t^{(14)}(x)\right)\right]$,
where $H_t^{(14)}(x)$ denotes the hidden representation of the key
entity's final token at layer~14
after $t$ training steps.\footnote{14 is a middle layer, where entity
representations tend to be encoded \citep{kaplan2025tokenswordsinnerlexicon}.}

\begin{wrapfigure}{r}{0.44\textwidth}
    \centering
    \vspace{-0.5em}
    \includegraphics[width=0.42\textwidth]{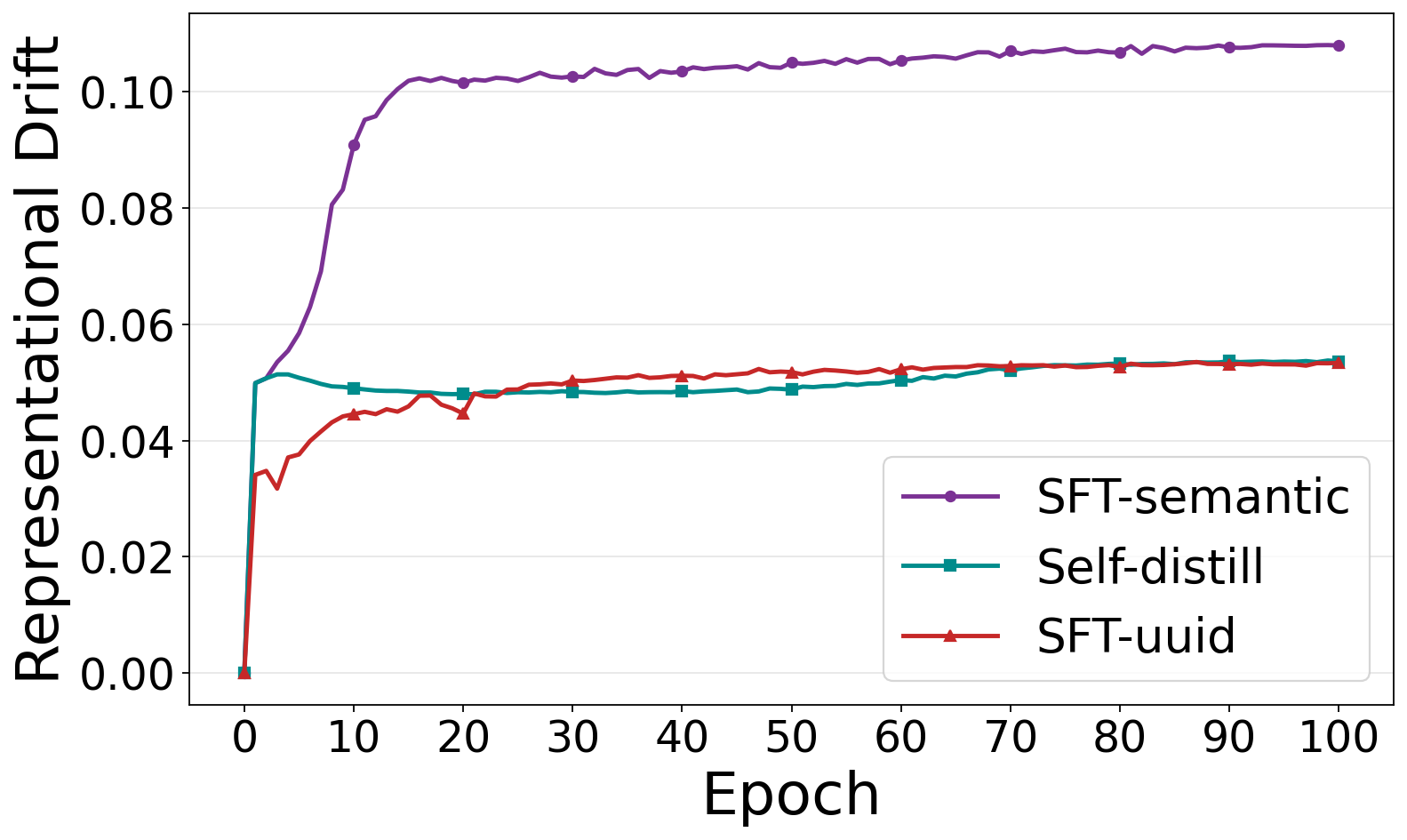}
    \caption{
\textbf{Hidden-state drift} ($\mathrm{RD}_t$; defined in text).
\textcolor{red}{UUID keys} and \textcolor{teal}{self-distillation} 
both stabilize near ${\approx}5\%$ after task-format learning; 
\textcolor{violet}{SFT on semantic keys} continues to ${\approx}11\%$, 
implicating representational interference as the source of factual forgetting.
    }
    \label{fig:hidden_drift}
    \vspace{-1em}
\end{wrapfigure}
\Cref{fig:hidden_drift} shows the drift trajectories for three conditions: SFT on semantic keys, self-distillation on semantic keys, and SFT on UUID keys. All three rise to ${\approx}5\%$ within the first epoch, a task-format learning signal present regardless of entity type. The curves then diverge: under SFT on semantic keys, drift continues to ${\approx}11\%$ as overlapping facts are acquired; under SFT on UUID keys, drift stabilizes at ${\approx}5\%$; and self-distillation on semantic keys also stabilizes near $5\%$.

The alignment between UUID and self-distillation is the key observation: UUID keys occupy a disjoint representational neighborhood, so learning them leaves nearby held-out representations undisturbed, while self-distillation reaches the same residual drift despite training on overlapping entity names. The excess drift from $5\%$ to $11\%$ under SFT on semantic keys is thus the component attributable to representational overlap: it is absent under both conditions that avoid or suppress interference, and is not predicted by capacity or behavioral alternatives.\footnote{Additional experiments on the distillation objective and drift metrics are reported in \appref{app:drift}.} 

Self-distillation thus localizes the origin of forgetting by removing it: suppressing output-distribution drift appears to eliminate the overlap-induced component.


\subsection{Why Overlap Predicts Forgetting: An Associative-Memory Account}
\label{sec:am}

Our experiments show that overlap governs forgetting; they do not say \textit{why} it should. We close with an account resting on a single premise: new facts enter the same store as the old ones. Following \citet{nichani2025understandingfactualrecall}, we treat factual retrieval as an effective associative memory that pairs each subject's internal representation, its \emph{key}, with the output direction of its answer. Writing $k_x\in\mathbb{R}^d$ for the key of subject $x$ and $u_y\in\mathbb{R}^d$ for the unembedding of answer $y$, the store and its readout take the form
\begin{equation}
W \;=\; \sum_{x\in\mathcal{X}} u_{f^\star(x)}\, k_x^{\top},
\qquad
s_y(x) \;:=\; u_y^{\top} W\, k_x ,
\label{eq:W}
\end{equation}
where the sum runs over the stored subjects $\mathcal{X}$, $f^\star$ maps each to its stored answer, $W$ is an effective operator describing the network's input--output behavior, and $s_y(x)$ is the resulting logit for any candidate answer $y$. In \secref{semantic_overlap} the subjects are locations and the answers countries, and the model returns $\arg\max_y s_y(x)$: a stored fact is forgotten when a competitor overtakes its answer, as $s_{\textit{Greece}}(\textit{Milan})$ overtakes $s_{\textit{Italy}}(\textit{Milan})$. Keys are the representations whose drift we tracked in \secref{why_self-distillation}, and they are not orthogonal, so writes leak between stored facts~\citep{elhage2022toymodelssuperposition}.

SFT writes into this same memory. To first order in the learning rate $\eta$ at a fixed key map, a softmax cross-entropy step on a new fact with subject $x_j$ and answer $a_j$ adds a rank-one Hebbian increment~\citep{kohonen1972correlation,hopfield1982neural} $\Delta W_j=\eta\,(u_{a_j}-\bar u_j)\,k_{x_j}^{\top}$, with $\bar u_j$ the mean unembedding under the model's prediction (Lemma~\ref{lem:hebbian}, \appref{app:theory}). Its effect on the fact stored under a held-out subject $x_h\in\mathcal{D}_{\mathrm{Held}}$ factorizes:
\begin{equation}
\delta^{(j)}_y(x_h) \;:=\; u_y^{\top}\Delta W_j\,k_{x_h}
\;=\;
\underbrace{\eta\,\langle u_y,\;u_{a_j}-\bar u_j\rangle}_{\text{answer side}}
\;\cdot\;
\underbrace{\rho_{jh}}_{\text{key overlap}},
\qquad
\rho_{jh}:=\langle k_{x_j},k_{x_h}\rangle
\label{eq:factorize}
\end{equation}
(the cosine, for unit-norm keys). The stored fact enters only through $\rho_{jh}$, the overlap between the two subjects' keys: the answer sets how large the perturbation is; the overlap sets which stored facts receive it. Learning \emph{Bergadena}$\rightarrow$\emph{Greece} thus pushes the readout for \emph{Milan}$\rightarrow$\emph{Italy} toward \emph{Greece} in proportion to $\langle k_{\textit{Bergadena}},k_{\textit{Milan}}\rangle$, which is large; \texttt{Loc\_fcfb46ee}$\rightarrow$\emph{Greece} pushes just as hard against a negligible overlap~(\cref{fig:dynamics}, right).

Summing over the $M$ injections splits the perturbation into a bias and a mean-zero fluctuation (Proposition~\ref{prop:bias-var}). The bias stays bounded in $M$ and fades as the new fact is acquired. The scale dependence lives in the fluctuation, whose variance is bounded by a constant multiple of $\eta^2 S_h/d$, where $S_h:=\sum_{j=1}^{M}\rho_{jh}^2=M\,\overline{\rho_h^2}$ is the number of injections times their mean squared overlap with $k_{x_h}$. 
Writing $a_h:=f^\star(x_h)$ for the stored answer of $x_h$, let $m_h:=s_{a_h}(x_h)-\max_{y\neq a_h}s_y(x_h)$ be its pre-fine-tuning margin. Under a Gaussian approximation to the fluctuation (\appref{app:forget-law}),
\begin{equation}
\Pr[\text{forget } x_h]\;\approx\;g\!\left(\frac{m_h}{\eta\sqrt{S_h/d}}\right),
\qquad g \text{ decreasing},
\label{eq:forget-law}
\end{equation}
so forgetting rises with the accumulated overlap $S_h$ (interference) and the learning rate $\eta$, and falls with the margin $m_h$ (confidence) and the dimension $d$.

Because $S_h$ is the product of scale and overlap, \cref{eq:forget-law} predicts both axes of \secref{semantic_overlap}. UUID keys have the least overlap with real entities (\appref{app:geometry}), so $S_h$ grows slowly with $M$ and forgetting stays at 0--4 pp even at $M=10^6$, while for semantic keys both grow sharply (\cref{fig:scale-curves}). At fixed $M=10^4$, the base-model overlap ordering of the four key conditions is exactly the observed forgetting ordering (\cref{fig:dose-response}). Only the ranking enters, so the prediction is insensitive to the layer at which keys are read. The account also covers the value indifference of \cref{tab:semantic_overlap} and the saturation of \cref{fig:baseline}, and it explains why a KL penalty works where a weight-space penalty does not (\secref{why_self-distillation}): the interference lives along the unembedding directions, which KL constrains and a weight norm cannot see (\appref{app:output-vs-weight}). It is qualitative---\appref{app:theory} states its idealizations and extends it to distributed storage---so it predicts orderings, not a functional form. That ordering is visible before training, in the overlap between new and stored keys.

\section{Related Work}
\label{sec:background_related}

\paragraph{Hallucinations and scope}
Hallucinations are commonly defined as model outputs unfaithful to real-world facts~\citep{Ji_2023, simhi2025hackhallucinationscertaintyknowledge}, provided context~\citep{liu2025longcontexthallucinationdetection}, or user instructions~\citep{belem2025singlemultillmshallucinate}, arising from mechanisms spanning data, training, and inference~\citep{kalai2025languagemodelshallucinate, azaria2023internalstatellmknows}. 
We view SFT-induced hallucinations as the behavioral manifestation of
\emph{factual forgetting} driven by representational interference: in the closed-book QA format~\citep{gekhman2024doesfinetuningllmsnew, ovadia2024finetuningretrievalcomparingknowledge, zucchet2025languagemodelslearnfacts}, a previously known fact the model can no longer produce \emph{is} a hallucination. Our results suggest the phenomenon arises whenever fine-tuning disturbs prior facts.

\paragraph{Factual knowledge in LLMs and module roles}
LLMs encode factual associations parametrically, with both FFN layers \citep{geva2021transformerfeedforwardlayerskeyvalue, meng2023locatingeditingfactualassociations, kaplan2025tokenswordsinnerlexicon} and attention projections \citep{dar2023analyzingtransformersembeddingspace, elhelo2025inferringfunctionalityattentionheads, chughtai2024summingfactsadditive} storing and expressing it. Representations are thus distributed across components, and skill performance survives updating only parameter subsets~\citep{zhu2025teachlargemultimodalmodels}; \secref{Modules_ablations} shows how that choice matters.

\paragraph{Hidden knowledge and recall failures}
LLMs often encode more knowledge than they express, hallucinating while internally knowing the answer~\citep{gekhman2025insideouthiddenfactualknowledge, orgad2025llmsknowshowintrinsic, simhi2025hackhallucinationscertaintyknowledge,zhang2026reinforcementlearningimprovestraversal}, and latent knowledge can sometimes be elicited by reasoning~\citep{gekhman2026thinkingrecallreasoningunlocks, calderon2026shelveslostkeysrecall}. Some SFT-induced hallucinations may thus reflect degraded recall rather than loss: the interference we identify may corrupt recall pathways without erasing the knowledge, and self-distillation may preserve both.

\paragraph{Continual learning in LLMs}
Continual learning studies the stability-plasticity tension in sequential training, proposing regularization, replay, architectural isolation, and distillation~\citep{li2017learningforgetting, Kirkpatrick_2017, delange2021continuallearningsurveydefying}; recent work applies distillation and KL-based constraints to continual \emph{task} learning in LLMs~\citep{zhu2025teachlargemultimodalmodels, shenfeld2025rlsrazoronlinereinforcement}.
Closest to our method, SDFT~\citep{shenfeld2026selfdistillationenablescontinuallearning} distills on-policy with reverse KL from a teacher conditioned in-context on an expert demonstration, teaching a new behavior; our teacher is an unconditioned frozen snapshot, distilled off-policy with forward KL as a temporal anchor for prior knowledge. The two share a core mechanism while targeting complementary problems.
We treat preservation of \emph{pretrained factual knowledge} as the stability objective, in contrast to continual knowledge learning, whose goal is the factual update itself~\citep{jang2022continualknowledgelearninglanguage, chen2025continualmemorizationfactoidslanguage}.

\section{Conclusion}
We reframed SFT-induced hallucinations as factual forgetting arising from continual learning dynamics, and showed that established mitigation strategies transfer to this setting. When new fact acquisition is undesirable, selective freezing suppresses factual plasticity while preserving task learning; when it is required, self-distillation constrains output-distribution drift and reduces forgetting from $\sim$15 to 2 pp without sacrificing plasticity. Both mitigations carry over to realistic instruction tuning. We also found that forgetting is selective: it scales with the representational overlap between new and existing entities, while random UUID facts induce near-zero forgetting even at $10^6$ new facts, implicating localized interference rather than capacity or behavioral effects. Representational drift analysis suggests self-distillation suppresses this interference at the output level, and an associative-memory account predicts the forgetting from the overlap between new and stored keys.

In summary, SFT-induced hallucinations are not an inevitable cost of factual learning but a consequence of representational interference that can be targeted with selective freezing or self-distillation. We hope this encourages future work that treats factual stability as a first-class fine-tuning objective.

\section*{Acknowledgments}
We thank the anonymous reviewers and the area chair for their thorough and constructive feedback, which substantially shaped this version. We thank Uri Alon for his help refining the theoretical derivations during the discussion period.
This work was supported in part by the Israel Science Foundation (grant no.\ 2045/21), by NSF-BSF grant 2020793 and by the ONR award N00014-23-1-2383.

\section*{Reproducibility Statement}
All experimental configurations are specified in the paper and appendices. SLiCK classification (\secref{Methodology}) uses 20 evaluation runs per question, each with three randomly sampled few-shot exemplars, retaining only relations for which at least 30\% of examples are classified as \emph{HighlyKnown}. Training uses learning rate $5\times10^{-5}$ (selected by jointly evaluating time-to-learn on unknown facts and the induced hallucination rate), batch size 16, and the hyperparameters listed per experiment; self-distillation defaults are $\lambda=1$, $\tau=0.5$, teacher snapshot at epoch 1 (\appref{app:different_sd_params}). Synthetic-entity construction is described in \appref{app:synthetic_creation}, the Alpaca protocol in \appref{app:alpaca}, representation-geometry analyses in \appref{app:geometry}, and baseline sweeps in \appref{app:baselines}. Code and data-construction scripts are available at \url{https://github.com/schwartz-lab-NLP/finetuning-hallucinations}.

\paragraph{LLM usage.}
Large language models were used as general-purpose assistants in preparing this work: polishing prose (grammar, phrasing, and tightening for length), suggesting \LaTeX{} and plotting boilerplate, and assisting with refactoring and documentation of the released code. All research questions, experimental design, analyses, and scientific claims are the authors' own, and the authors verified every number, citation, and statement reported here. Separately, and as described in the text, LLMs appear within the experiments themselves---as the models under study, and as automated annotators and judges (\textsc{Llama}-3.3-70B-Instruct for Alpaca factuality annotation and AlpacaEval judging; \appref{app:alpaca}).

\section*{Note on this version}
This version differs from the COLM 2026 proceedings version in \secref{am}: the forgetting law, \cref{eq:forget-law}, and the margin $m_h$ are stated in the main text (in the proceedings version they appear only in \appref{app:forget-law}), and the closing paragraphs are condensed. Equation numbers after \cref{eq:factorize} are shifted by one. No results, figures, or tables changed.

\bibliography{colm2026_conference}
\bibliographystyle{colm2026_conference}

\appendix

\section{Results Over Other SLiCK Classification Groups}
\label{app:other_clf}

The main experiments focus on \emph{HighlyKnown} facts throughout, on both the training and validation sides. This choice follows a consistent principle. On the training side, the role of $\DKnown$ is to teach the model the QA task format without introducing any new factual content. Facts classified as \emph{HighlyKnown} by SLiCK, meaning those for which the model produces the correct answer as its top-1 prediction across all few-shot prompt configurations, are ideally suited for this purpose: training on them reinforces format while leaving factual representations undisturbed. Any other group would risk conflating format acquisition with factual learning, obscuring the mechanism under study. On the validation side, having restricted $\DKnown$ to \emph{HighlyKnown} facts, the natural corresponding held-out set is drawn from the same group, providing a clean forgetting signal: any accuracy decline on $\DHeld$ reflects genuine factual interference rather than pre-existing encoding fragility.

The four relations retained for our experiments (P17, P36, P407, and P495) were selected in part because each exhibits \emph{HighlyKnown} prevalence exceeding 30\% on the development split, ensuring a sufficient pool for both $\DKnown$ and $\DHeld$. \Cref{tab:app_slick_dist} reports the full SLiCK category distribution across all four relations on the development set ($n{=}3{,}530$). \emph{HighlyKnown} facts constitute the plurality (46.6\%), followed by \emph{MaybeKnown} (32.4\%); \emph{WeaklyKnown} (5.3\%) and \emph{Unknown} (15.8\%) make up the remainder. \emph{WeaklyKnown} is the smallest group: these are facts the model never answers correctly as its greedy top-1 prediction in any of SLiCK's 20 few-shot prompt variations, the lowest degree of factual knowledge.

\begin{table}[h]
\centering
\footnotesize
\setlength{\tabcolsep}{4pt}
\begin{tabular}{@{}lcccc@{}}
\toprule
\textbf{Relation} & \textbf{HighlyKnown} & \textbf{MaybeKnown} & \textbf{WeaklyKnown} & \textbf{Unknown} \\
\midrule
P17  (location $\rightarrow$ country)       & 41.3\% & 44.3\% & 6.5\%  & 7.9\%  \\
P36  (country $\rightarrow$ capital)        & 50.7\% & 9.3\%  & 4.4\%  & 35.6\% \\
P407 (artwork $\rightarrow$ language)         & 42.0\% & 45.4\% & 4.7\%  & 7.9\%  \\
P495 (artwork $\rightarrow$ origin country)   & 51.1\% & 32.5\% & 5.2\%  & 11.2\% \\
\midrule
\textbf{Combined} & \textbf{46.6\%} & \textbf{32.4\%} & \textbf{5.3\%} & \textbf{15.8\%} \\
\bottomrule
\end{tabular}
\caption{
\textbf{SLiCK category distribution on the development split for the four selected relations} ($n{=}3{,}530$). Each relation exceeds 30\% \emph{HighlyKnown}, the threshold used for inclusion.
}
\label{tab:app_slick_dist}
\end{table}

While the main text focuses exclusively on \emph{HighlyKnown}, the held-out set also contains \emph{MaybeKnown} and \emph{WeaklyKnown} facts, which exhibit qualitatively distinct dynamics under fine-tuning. \Cref{fig:val_categories} reports training curves for all three knowledge groups under Regular SFT, Self-Distillation, and Only Known.

\begin{figure}[h]
\centering
\includegraphics[width=\textwidth]{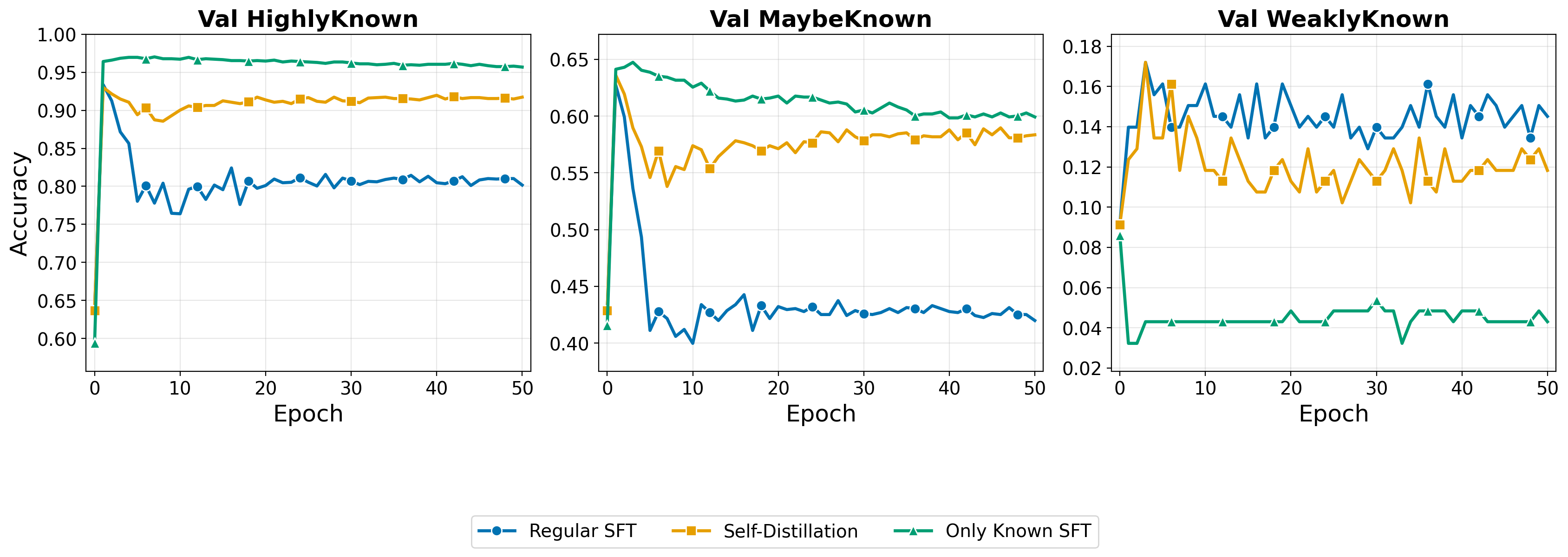}
\caption{
\textbf{Training dynamics across SLiCK knowledge groups.} Factual correctness over training epochs on held-out \emph{HighlyKnown} (left), \emph{MaybeKnown} (middle), and \emph{WeaklyKnown} (right) facts, for Regular SFT (blue), Self-Distillation (orange), and Only Known (green).
}
\label{fig:val_categories}
\end{figure}

The \emph{HighlyKnown} panel (left) reproduces the $\DHeld$ results from the main text. Only Known preserves accuracy throughout (${\approx}95\%$), self-distillation stabilizes at ${\approx}92\%$, and regular SFT degrades to ${\approx}80\%$, a drop of roughly 13 pp from the pre-fine-tuning baseline.

\emph{MaybeKnown} facts (middle panel) reveal substantially more severe forgetting under regular SFT. These facts are encoded less robustly: the model produces the correct answer under some but not all SLiCK prompt configurations, and this pre-existing fragility amplifies their susceptibility to representational interference. Accuracy collapses from approximately 0.64 to 0.43, a drop of roughly 21 pp, compared to 13 pp on \emph{HighlyKnown}. Self-distillation provides a commensurately larger benefit, stabilizing at approximately 0.58 and recovering approximately 15 pp over regular SFT. Only Known again provides the strongest preservation (${\approx}0.60$), confirming that $\DUnk$ is the proximate driver of interference regardless of the robustness of the prior encoding. This pattern is consistent with the findings reported for Llama-3.1-8B and Qwen2.5-7B in \appref{app:different_models}.

\emph{WeaklyKnown} facts (right panel) exhibit a qualitatively different pattern that inverts the ordering observed in the other two groups. These facts are those for which the model never produces the correct answer as its top-1 prediction under any SLiCK configuration (knowledge that is present but highly latent), and baseline accuracy is accordingly low (${\approx}0.15$--$0.17$). Strikingly, Only Known drops to approximately 0.04--0.05 and remains there throughout training, well below both Regular SFT (${\approx}0.14$) and Self-Distillation (${\approx}0.13$). Training exclusively on \emph{HighlyKnown} facts appears to further suppress recall of \emph{WeaklyKnown} ones: by reinforcing a retrieval regime calibrated for robustly encoded facts, it may render the more effortful access required for latent knowledge less accessible. Regular SFT on the full training mixture (including $\DUnk$) provides a more varied optimization signal that may incidentally preserve or lightly activate recall of weakly encoded facts. Self-distillation and regular SFT converge to similar performance in this regime, which is expected: the distributional constraint provides little additional benefit when the facts in question were not reliably accessible before fine-tuning began.

\section{Results Across Different Models}
\label{app:different_models}

We validate our main findings on two additional 7--8B-parameter models from distinct families: Llama-3.1-8B and Qwen2.5-7B. \Cref{fig:app_models_grid} reports training dynamics across all three conditions (Regular SFT, Self-Distillation, Only Known) for both architectures.

\begin{figure}[h]
\centering
\includegraphics[width=0.7\textwidth]{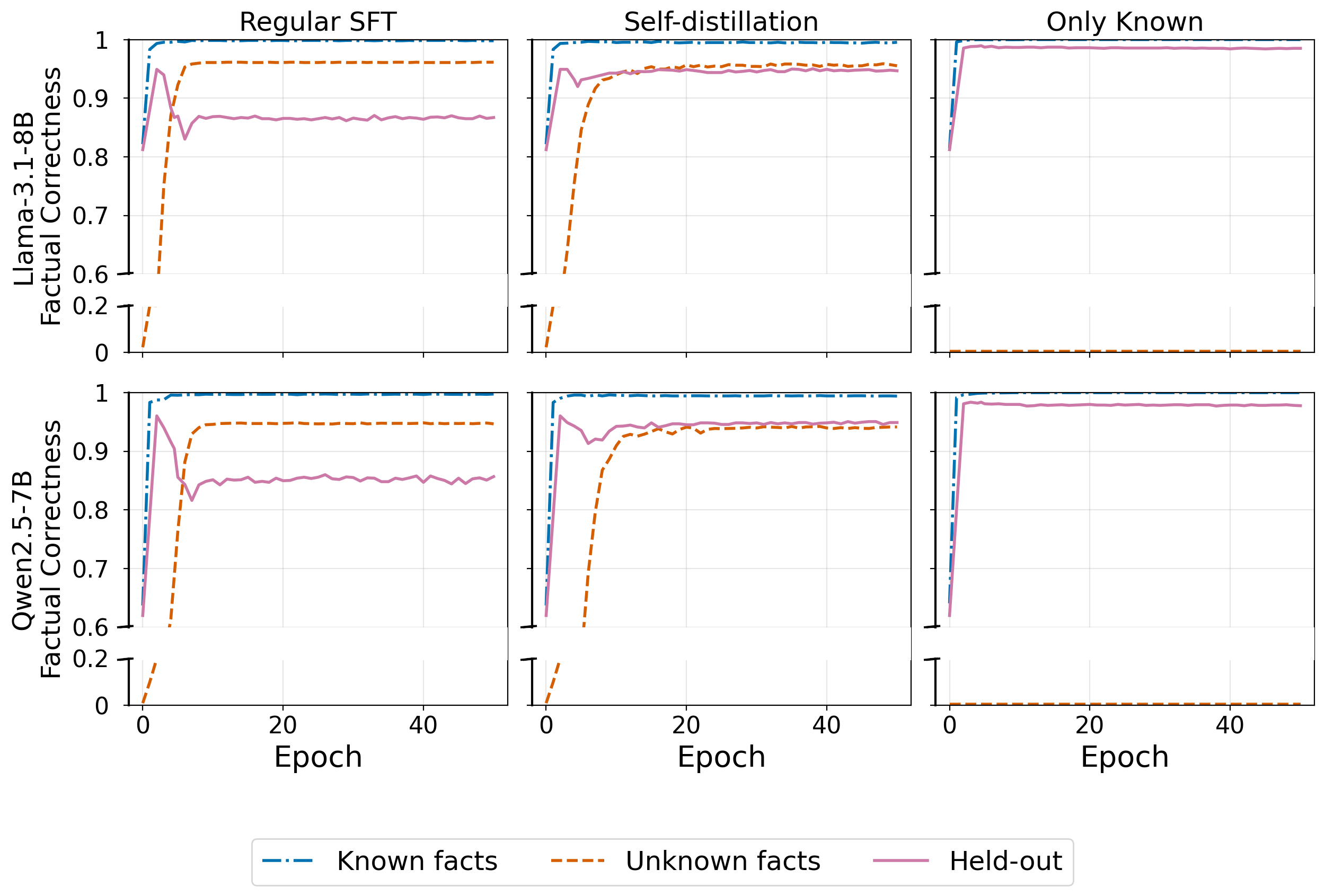}
\caption{
\textbf{Training dynamics on Llama-3.1-8B (top) and Qwen2.5-7B (bottom).}
Each panel shows factual correctness over training epochs for \KnownFacts{} (blue dash-dot), \UnkFacts{} (orange dashed), and \HeldFacts{} (pink solid), across Regular SFT (left), Self-Distillation (middle), and Only Known (right).
}
\label{fig:app_models_grid}
\end{figure}

The core training dynamics replicate faithfully across both model families, consistent in magnitude with our primary 1.5B-parameter experiments. Under regular SFT, both Llama-3.1-8B and Qwen2.5-7B exhibit the characteristic two-stage pattern: rapid task acquisition in the first few epochs---with near-perfect accuracy on $\DKnown$ (97--98\%)---followed by progressive forgetting of $\DHeld$ to approximately 85--86\%. The Only Known condition preserves $\DHeld$ throughout training at 95--97\%, and self-distillation closely tracks this upper bound, stabilizing at approximately 95\%---a drop of only 2--3 pp from the Only Known baseline, compared to 11--12 pp under regular SFT. This corresponds to approximately 80\% mitigation of forgetting, consistent with the degree of protection observed in the 1.5B model, confirming that self-distillation generalizes reliably across architectures and scales. The residual degradation under self-distillation ($\sim$2--3pp) is likely attributable to task-format adaptation rather than factual interference, consistent with the $\sim$5\% baseline drift observed even in the UUID condition (\secref{semantic_overlap}).

A complementary observation emerges from the fine-grained freezing results in \appref{app:different_modules}: training the full attention block is insufficient to achieve forgetting mitigation. The same level of stability is recovered only when restricting updates to a \emph{single} attention projection (e.g., $k$, $v$, or $o$), with all other parameters frozen. This finding goes beyond the coarse attention-vs-FFN distinction established in \secref{Modules_ablations}: within the attention block itself, granularity matters. Updating all attention parameters jointly preserves enough factual plasticity to drive forgetting; restricting to a single projection suppresses it enough to protect $\DHeld$.

\begin{figure}[h]
\centering
\includegraphics[width=0.7\textwidth]{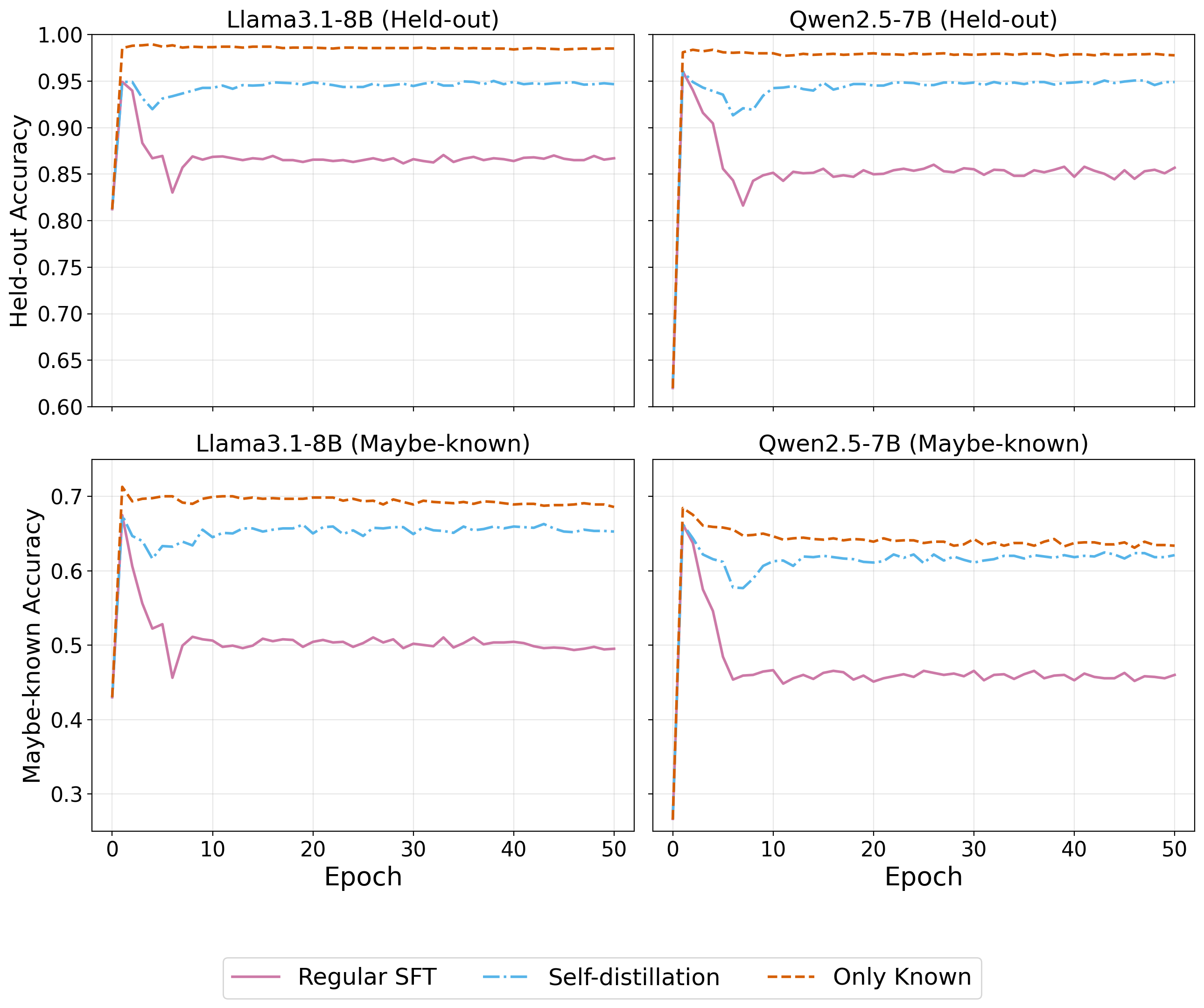}
\caption{
\textbf{Accuracy on $\DHeld$ and \emph{MaybeKnown} facts across training conditions.}
Top row: factual correctness on held-out \emph{HighlyKnown} facts ($\DHeld$). Bottom row: accuracy on out-of-training \emph{MaybeKnown} facts. Results are shown for Llama-3.1-8B (left column) and Qwen2.5-7B (right column) under Regular SFT (solid pink), Self-Distillation (blue dash-dot), and Only Known (orange dashed).
}
\label{fig:app_maybe_known}
\end{figure}

Beyond the held-out known facts, we examine the effect on \emph{MaybeKnown} facts, which SLiCK classifies as only partially known: the model produces the correct answer inconsistently across prompting configurations. This pattern is consistent across both model families and sizes (\cref{fig:app_maybe_known}), and the effect is substantially more pronounced than what is observed for $\DHeld$. Under regular SFT, accuracy on \emph{MaybeKnown} facts drops by approximately 18--19 pp (from $\sim$0.63 to $\sim$0.45 for Llama-3.1-8B, and from $\sim$0.68 to $\sim$0.49 for Qwen2.5-7B), compared to an 11--12 pp drop on $\DHeld$. We attribute this amplified vulnerability to the nature of partial knowledge: \emph{MaybeKnown} facts are encoded less robustly in the model's representations, and weaker encoding makes them more susceptible to interference when new factual knowledge is integrated. Their recall is already fragile before fine-tuning, so even moderate representational drift is sufficient to push them below the threshold of reliable recall. Self-distillation provides a commensurately larger benefit in this regime, stabilizing \emph{MaybeKnown} accuracy at approximately 62--65\% and mitigating approximately 90\% of the forgetting induced by regular SFT---compared to approximately 80\% on $\DHeld$---underscoring its particular importance for preserving knowledge that sits at the margins of model recall.

\section{Results Over Different Module Freezing Experiments}
\label{app:different_modules}

The main text (\secref{Modules_ablations}) establishes that freezing the FFN and updating only attention layers reduces forgetting by suppressing new fact acquisition. Here we examine the mechanism more finely: \emph{parameter freezing is not special}; it is simply one way to reduce factual plasticity. What drives the reduction in forgetting is not the act of freezing per se, but the resulting decrease in $\DUnk$. We verify this by sweeping over fine-grained parameter subsets and showing that the relationship between plasticity and forgetting is monotone and consistent regardless of which module is restricted.

\begin{table}[h]
\centering
\small
\begin{tabular}{lccc}
\toprule
\textbf{$\theta_S$} & $\DUnk\downarrow$ & $\DKnown\uparrow$ & $\DHeld\uparrow$ \\
\midrule
$k$        & 0.006 & 0.944 & 0.927 \\
$q$        & 0.006 & 0.912 & 0.901 \\
$v$        & 0.005 & 0.920 & 0.901 \\
$o$        & 0.006 & 0.941 & 0.925 \\
attn       & 0.010 & 0.946 & 0.931 \\
\midrule
gate+up    & 0.242 & 0.977 & 0.883 \\
down       & 0.086 & 0.956 & 0.918 \\
FFN        & 0.941 & 0.997 & 0.782 \\
\midrule
All (standard SFT)  & 0.946 & 0.990 & 0.780 \\
All (Only Known)    & ---   & 0.999 & 0.958 \\
\bottomrule
\end{tabular}
\caption{
\textbf{Fine-grained freezing ablations.} Each row corresponds to a different updated parameter subset $\theta_S$, with all other parameters frozen. $\DHeld$ decreases monotonically with $\DUnk$ across all configurations, confirming that factual plasticity---not the specific parameter group---mediates forgetting. Attention projections and the full attention block suppress plasticity most effectively. Within the FFN, gate+up drives substantially more plasticity than down. Full FFN and full model (standard SFT) recover the baseline forgetting pattern. Only Known provides the stability upper bound.
}
\label{tab:app_freezing}
\end{table}

\Cref{tab:app_freezing} reports results for individual attention projections ($q,k,v,o$), FFN sub-components (gate+up, down), the full attention block (attn), full FFN, and the standard full-model baseline. Across all configurations, $\DHeld$ tracks $\DUnk$ closely: as factual plasticity increases, forgetting increases proportionally, irrespective of which parameter group is updated. Individual attention projections suppress plasticity most aggressively ($\DUnk \approx 0.005$--$0.006$), yielding the highest $\DHeld$ values. The full attention block achieves comparable stability ($\DUnk = 0.010$), confirming that any restriction that prevents factual integration is sufficient. Within the FFN, the gate+up sub-component drives substantially more plasticity ($\DUnk = 0.242$) than the down projection ($\DUnk = 0.086$), consistent with the role of the gate+up pathway in writing factual content into the residual stream~\citep{geva2022transformerfeedforwardlayersbuild}. Updating the full FFN or the full model recovers standard SFT dynamics, with high plasticity and a corresponding 15-percentage-point drop in $\DHeld$.
\begin{wrapfigure}{r}{0.60\textwidth}
    \centering
    \includegraphics[width=0.60\textwidth]{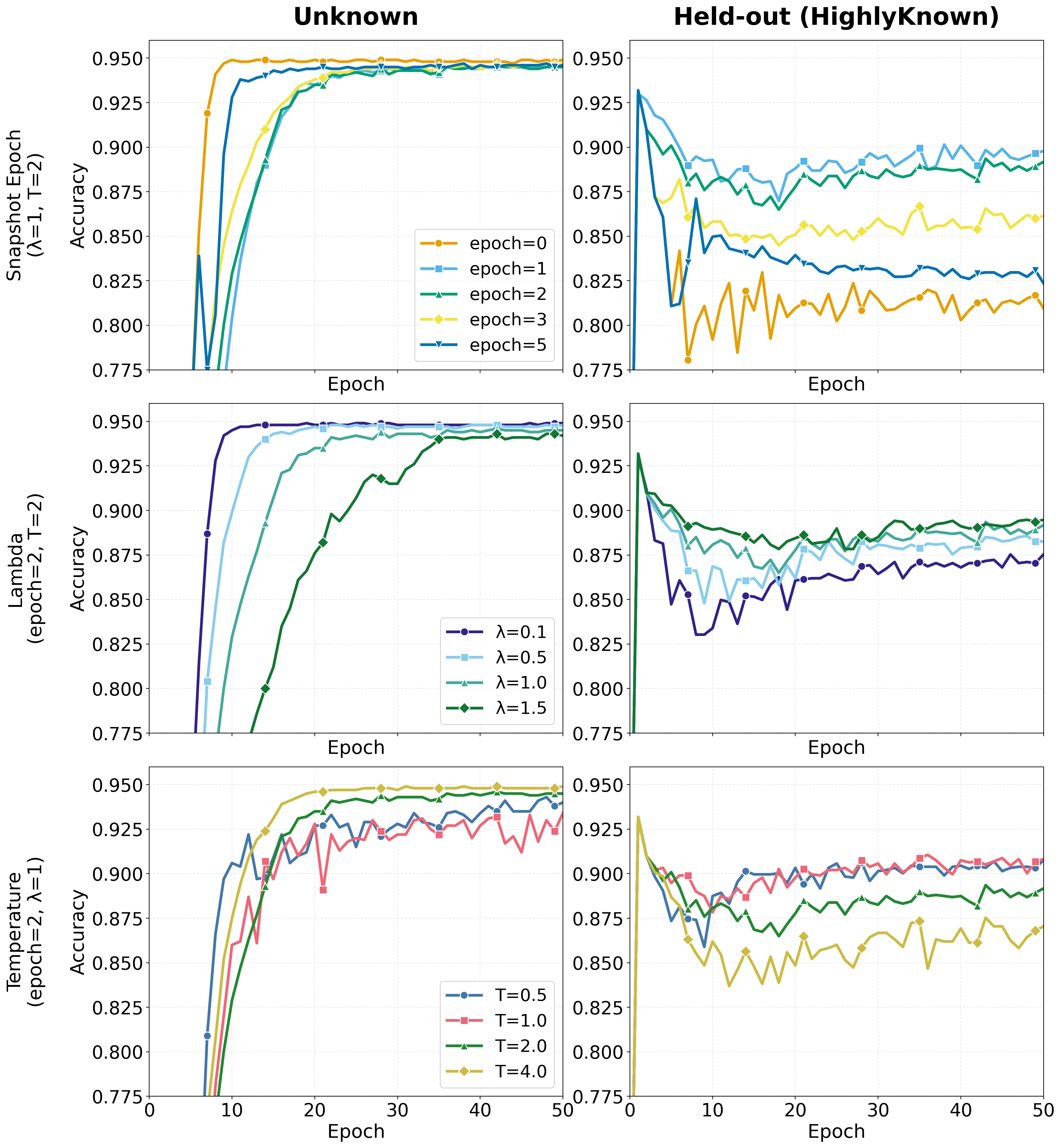}
    \caption{
    \textbf{Hyperparameter ablations for self-distillation.} Each row ablates one hyperparameter with the others fixed (see text). Left: $\DUnk$ accuracy (factual plasticity). Right: $\DHeld$ accuracy (factual stability).
    }
    \label{fig:app_sd_hparams}
    \vspace{-6em}
    \end{wrapfigure}

These results reinforce the central claim of \secref{Modules_ablations}: the stability--plasticity tradeoff in factual fine-tuning is governed by how much new factual content the model integrates, not by the structural form of the constraint. Freezing is one mechanism for controlling that integration; self-distillation (\secref{self-distillation}) is another, achieving the same stability without sacrificing plasticity.

\section{Results Over Different Parameters for Self-distillation}
\label{app:different_sd_params}

Self-distillation introduces three hyperparameters: the snapshot epoch $i$ at which the teacher is frozen, the distillation weight $\lambda$, and the temperature $\tau$. \Cref{fig:app_sd_hparams} reports ablations over each parameter independently: the snapshot-epoch sweep is run at $\lambda=1$, $\tau=2$, and the $\lambda$ and $\tau$ sweeps at snapshot epoch $i=2$ (with $\tau=2$ and $\lambda=1$, respectively). Each row shows accuracy on $\DUnk$ (left, factual plasticity) and $\DHeld$ (right, factual stability) over training epochs.

\paragraph{Snapshot epoch.}
When the teacher is frozen at $i=0$, $\DUnk$ acquisition is substantially slower than under all other snapshot choices, and $\DHeld$ is the lowest across the ablation.
In this setting, the distillation loss constrains the student toward a parameter space that predates task adaptation; since the model has yet to internalize the QA format, the lower $\DHeld$ may reflect difficulty in task learning rather than factual forgetting per se. Snapshot $i=1$ yields the highest $\DHeld$ values, with the model already adapted to the task format but not yet subject to sustained factual drift. Beyond $i=1$, held-out stability degrades monotonically with the snapshot index.

\paragraph{Distillation weight.}
The $\lambda$ ablation reveals a direct stability--plasticity tradeoff. Weaker regularization ($\lambda=0.1$, $\lambda=0.5$) allows rapid acquisition of $\DUnk$ but results in greater degradation of $\DHeld$. Stronger regularization ($\lambda=1.5$) substantially reduces forgetting but delays factual learning, requiring considerably more training epochs before $\DUnk$ accuracy reaches its asymptote. We select $\lambda=1$ as it achieves efficient factual acquisition alongside meaningful forgetting mitigation, without demanding an extended training budget.

\paragraph{Temperature.}
Higher temperatures ($\tau=2$, $\tau=4$) produce smoother and more stable learning trajectories on $\DUnk$, reducing epoch-to-epoch variance. However, $\DHeld$ is noticeably lower under high-$\tau$ conditions. Lower temperatures ($\tau=0.5$, $\tau=1$) yield better stability on $\DHeld$ while preserving comparable final accuracy on $\DUnk$. We adopt $\tau=0.5$ as our default; practitioners who prioritize stable factual learning trajectories may prefer higher temperatures at a modest additional forgetting cost. The best-performing point of the sweep itself is ($i=2$, $\lambda=1$, $\tau=1$). The default configuration used throughout the paper ($i=1$, $\lambda=1$, $\tau=0.5$) retains $\DHeld=0.908$ at epoch 50, about 2 pp below the $0.93$ peak; this is the value reported for self-distillation in \cref{tab:baselines} and the source of the figure quoted in \secref{self-distillation}.

\section{Parameter-Space Baselines: Protocols and Sweeps}
\label{app:baselines}

All baselines run in the 8K \textsc{EntityQuestions} setting of \secref{setting} (50/50 Known/Unknown mixture, learning rate $5\times10^{-5}$, identical supervision and evaluation), trained to convergence. \Cref{tab:baselines} summarizes each method at its best-performing hyperparameter.

\begin{table}[h]
\centering
\small
\begin{tabular}{lcc}
\toprule
Method & Final $\DHeld$ & Final $\DUnk$ \\
\midrule
Standard SFT & 0.780 & 0.946 \\
Weight decay (best strength: 50) & 0.804 & 0.909 \\
$\ell_2$ toward $\theta_i$ & 0.809 & 0.949 \\
EWC ($\DKnown$ Fisher proxy, best $\lambda$) & 0.826 & 0.948 \\
LoRA (best rank: 64) & 0.834 & 0.948 \\
Self-distillation (ours) & \textbf{0.908} & 0.934 \\
\bottomrule
\end{tabular}
\caption{\textbf{Output-space vs.\ parameter-space regularization} in the 8K EntityQuestions setting (peak $\DHeld \approx 0.93$). Each baseline at its best-performing hyperparameter.}
\label{tab:baselines}
\end{table}

\paragraph{EWC.}
Standard EWC~\citep{Kirkpatrick_2017} requires the old-task data (for the Fisher information), and SI~\citep{zenke2017continuallearningsynaptic} requires the training trajectory (for online importance accumulation); neither is accessible for pretraining, which rules out SI entirely and requires a proxy for EWC. We therefore estimate the empirical Fisher on a sample of $\DKnown$ as a stand-in for the pretrained factual knowledge. Across regularization strengths, EWC reproduces the $\ell_2$ pattern: the best-balanced setting reaches final $\DHeld = 0.826$ at full acquisition ($\DUnk = 0.948$), while stronger regularization ($\lambda = 10^5$) raises $\DHeld$ to $0.869$ only by collapsing acquisition ($\DUnk = 0.444$).

\paragraph{LoRA.}
Low-rank adaptation~\citep{hu2021loralowrankadaptationlarge} across ranks $\{16, 64\}$ applied to all modules. Consistent with \citet{biderman2024loralearnsforgets}, LoRA forgets somewhat less than full fine-tuning on a slower trajectory, but the pattern matches full SFT: final $\DHeld = 0.814$ (rank 16) and $0.834$ (rank 64) at $\DUnk \approx 0.95$.

\paragraph{Weight decay.}
AdamW decoupled weight decay across strengths $\{0.01, 0.1, 1, 10, 50, 100\}$. Moderate strengths leave forgetting essentially unchanged (final $\DHeld = 0.78$--$0.82$); the strongest setting (100) destroys training entirely. The best-performing strength (50) reaches $\DHeld = 0.804$ at $\DUnk = 0.909$.

\paragraph{$\ell_2$ toward $\theta_i$.}
The snapshot-anchored $\ell_2$ penalty of \secref{why_self-distillation} with $\lambda$ matched to self-distillation in gradient magnitude: final $\DHeld = 0.809$ at $\DUnk = 0.949$.

\paragraph{Summary.}
Every parameter-space method lands in the $0.78$--$0.84$ range on final $\DHeld$ (\cref{tab:baselines}), a gain of at most $\sim$5pp over standard SFT versus $\sim$13pp for self-distillation, and each faces the same stability--plasticity tradeoff when its regularization strength increases. The self-distillation row uses the paper's default configuration ($i=1$, $\lambda=1$, $\tau=0.5$; \appref{app:different_sd_params}), evaluated at epoch 50. This is the pattern predicted by the associative-memory account (\secref{am}): weight-space penalties bound the update magnitude uniformly, without regard to the representational overlap $\rho_{jA}$ that determines which stored facts are at risk.

\section{Real-Data Validation on Alpaca: Protocol and Trajectories}
\label{app:alpaca}

\paragraph{Data and annotation.}
We use 15K examples from the Alpaca dataset~\citep{taori2023alpaca}. Because factual content in instruction-tuning data is implicit (woven into instructions and responses rather than stated as clean relational facts), we first pass each sample through an LLM annotator (\textsc{Llama}-3.3-70B-Instruct via TogetherAI) that decides whether the sample carries factual knowledge and, if so, extracts the underlying (subject, relation, object) content. For each factual sample we generate an \textsc{EntityQuestions}-style QA probe, and apply SLiCK (\secref{Methodology}) with \textsc{Qwen}~2.5-1.5B to classify the model's knowledge state for that fact. A set of 1{,}000 factual probes is held out for evaluation; the held-out factual accuracy reported in \cref{tab:alpaca} is computed over the \emph{HighlyKnown} subset of these probes. We emphasize that no output of this pipeline is used during training: the mitigation methods receive only the raw Alpaca data, and SLiCK serves purely as a diagnostic.

\paragraph{Training and evaluation.}
All three configurations fine-tune \textsc{Qwen}~2.5-1.5B for five epochs with learning rate $3\times10^{-5}$ and batch size 16. Self-distillation uses the same hyperparameters as in \secref{self-distillation} with no additional tuning; attention-only freezing follows \secref{Modules_ablations}. After every epoch we evaluate (i) held-out factual accuracy and (ii) alignment quality via AlpacaEval~\citep{alpaca_eval} on 200 instructions, judged by \textsc{Llama}-3.3-70B-Instruct via TogetherAI against the dataset's reference outputs; the preferred rate counts wins plus half of ties. \Cref{tab:alpaca} reports AlpacaEval at a common early-stopped epoch (4), the epoch selected by validation preferred rate for the two fully-trained configurations (standard SFT and self-distillation); we apply the same epoch to attention-only freezing for comparability, as its trajectory is essentially flat (its own best preferred rate, $50.3\%$, occurs at epoch 2). Factual degradation is reported at the end of training, measured against the peak held-out accuracy observed across all three runs ($0.869$, attained under standard SFT at epoch 3): $0.869 \rightarrow 0.795$ (SFT), $0.852$ (self-distillation), and $0.860$ (attention-only), i.e.\ $8$, $2$, and $1$ pp at the two-decimal resolution appropriate for the $229$-probe evaluation set (one probe $\approx 0.4$pp). AlpacaEval on 200 instructions has run-to-run variability of roughly $\pm 1$--$2$pp: the same base model scored $39.8\%$, $38.8\%$, and $37.8\%$ across the three runs' initial evaluations, so differences of this order should be read with care.

\paragraph{Trajectories.}
Held-out factual accuracy over epochs 0--5 (epoch 0 = base model):
standard SFT $0.821 \rightarrow 0.821 \rightarrow 0.843 \rightarrow 0.869 \rightarrow 0.803 \rightarrow 0.795$;
self-distillation $0.821 \rightarrow 0.817 \rightarrow 0.834 \rightarrow 0.847 \rightarrow 0.843 \rightarrow 0.852$;
attention-only freezing $0.817 \rightarrow 0.856 \rightarrow 0.852 \rightarrow 0.852 \rightarrow 0.847 \rightarrow 0.860$.
The pattern matches the controlled setting: standard SFT peaks early and then degrades as factual content is absorbed, while both mitigations remain stable throughout. AlpacaEval preferred rates at epoch 4 are $50.5\%$ (SFT), $52.3\%$ (self-distillation), and $47.2\%$ (attention-only), against a base-model rate of $39.8\%$.

\section{Synthetic Entities Creation Procedure}
\label{app:synthetic_creation}

\begin{figure}[t]
\centering
\includegraphics[width=0.88\textwidth]{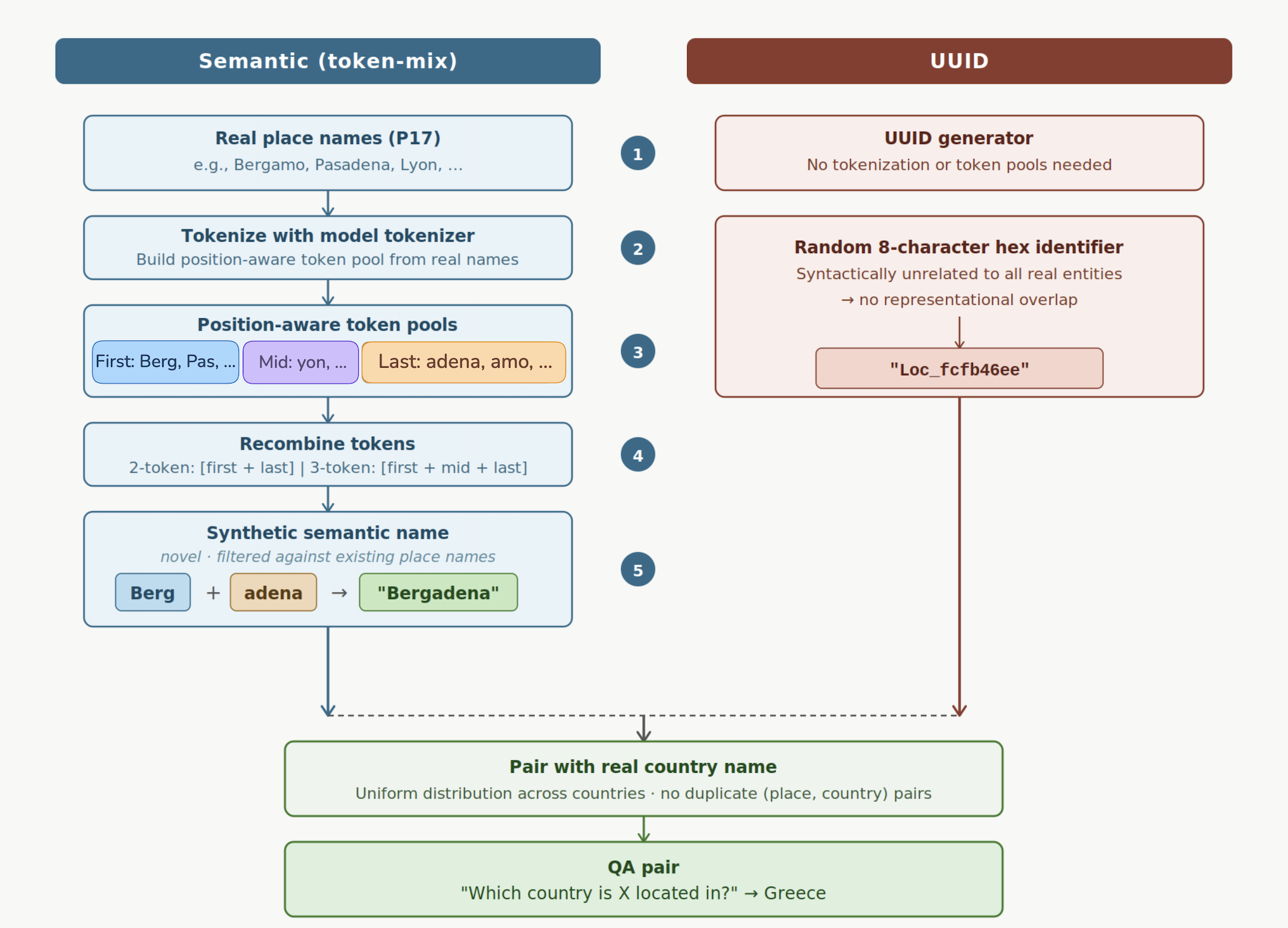}
\caption{
\textbf{Synthetic entity construction.}
\emph{Semantic (token-mix) keys} (left, steps 1--5): real P17 place names are tokenized, position-aware token pools are assembled, and novel names are formed by recombining tokens across names (e.g., \textbf{Bergadena} from \textbf{Berg}amo and Pas\textbf{adena}).
\emph{UUID keys} (right): random 8-character hex identifiers with no lexical overlap with real entities.
Both conditions are paired with real country names under an identical question template, so the entity key is the sole variable.
Candidate synthetic names that match any existing real place name are discarded.
}
\label{fig:synthetic_illustration}
\end{figure}

All controlled experiments in \secref{semantic_interference} use the P17 Wikidata relation (\emph{location} $\rightarrow$ \emph{country}: ``Which country is \emph{X} located in?''). We focus on a single relation deliberately: fixing the relation holds the question template, the answer domain, and the relational structure constant across all conditions, making the entity name (the \emph{key}) the sole experimental variable and the comparison between key types experimentally controlled.

To construct synthetic facts, we generate two types of entity names (\cref{fig:synthetic_illustration}). \textbf{Semantic (token-mix) keys} are built by tokenizing all real P17 place names and extracting three position-aware token pools: the \emph{first} token of each multi-token name, all \emph{middle} tokens, and the \emph{last} token. Novel place names are then formed by recombining tokens sampled from these pools, either as two-token sequences \texttt{[first, last]} or three-token sequences \texttt{[first, middle, last]}, and any candidate that coincides with an existing real place is discarded. The resulting names share sub-word structure with real locations: for example, tokens from \textbf{Berg}amo and Pas\textbf{adena} combine into \textbf{Bergadena}, a novel name that lexically resembles genuine place names and is thus expected to activate overlapping representations in the model. \textbf{UUID keys}, by contrast, are random identifiers of the form \texttt{Loc\_\textlangle{}8-hex\textrangle{}} (e.g., \texttt{Loc\_fcfb46ee}), with no syntactic resemblance to any real entity.
For the dose-response conditions of \secref{semantic_overlap}, we additionally generate \textbf{token-mix $k$-noise} keys ($k\in\{1,2\}$) by replacing $k$ of the tokens in each token-mix name with arbitrary identifier tokens drawn from the pool used for UUID construction, respecting token position; all other aspects of the construction (value assignment, filtering, balancing) are identical to the token-mix condition. In both conditions, each synthetic entity is paired with a real country drawn uniformly at random from those appearing in P17, subject to the constraint that the (place, country) pair is not already present in the original dataset; country assignments are rebalanced to maintain an approximately uniform distribution.

To verify that forgetting is driven by the key rather than the value, we additionally varied the country (value) condition: values were instantiated as real country names, semantic synthetic country names (constructed by the same token-mix process applied to country names), or UUID-style country labels, crossed with both key types. \Cref{tab:semantic_overlap} reports $\Delta\DHeld$ for all six combinations after training on $10^6$ synthetic facts. Across all value conditions, semantic keys consistently induce 38--41 pp of forgetting, while UUID keys yield approximately 4 pp forgetting regardless of what value is paired with them. This full dissociation indicates that, among the factors varied here, representational overlap in the entity name is the dominant driver of interference, while the answer side plays little role.

\begin{table}[h]
\centering
\small
\begin{tabular}{llc}
\toprule
\textbf{Key type} & \textbf{Value type} & $\Delta\DHeld\text{ (pp)}\downarrow$ \\
\midrule
Semantic & Real      & 38 \\
Semantic & Semantic  & 38 \\
Semantic & UUID      & 41 \\
\midrule
UUID     & Real      & 4 \\
UUID     & Semantic  & 4 \\
UUID     & UUID      & 4 \\
\bottomrule
\end{tabular}
\caption{
\textbf{Forgetting is driven by key semantics, not value type.}
$\Delta\DHeld$ after training on $10^6$ synthetic facts, for all combinations of key type (semantic vs.\ UUID) and value type (real / semantic / UUID country names). All conditions achieve $\DUnk = 100\%$ and $\DKnown \ge 90\%$, so differences in forgetting cannot be attributed to learning success.
}
\label{tab:semantic_overlap}
\end{table}

\section{Representational Geometry of Synthetic Keys}
\label{app:geometry}

\paragraph{Setup.}
For each key condition (real cities, token-mix, token-mix 1-noise, token-mix 2-noise, UUID) we extract residual-stream representations of $n=5{,}000$ entities per condition from the \emph{base} \textsc{Qwen}~2.5-1.5B model (no fine-tuning), at every layer, under two aggregations: the mean over the entity's sub-word tokens and the last sub-word token. Results below use the last-sub-word aggregation, which matches how key representations are read in \secref{why_self-distillation}; the mean aggregation agrees at the early and final layers but saturates mid-network, where all conditions approach ceiling similarity. Two additional runs ($n=2{,}000$, two seeds) replicate the cosine ordering at every layer; the probe ordering at this smaller scale is reproduced in one seed and flat in the other.

\begin{figure}[t]
\centering
\includegraphics[width=0.85\textwidth]{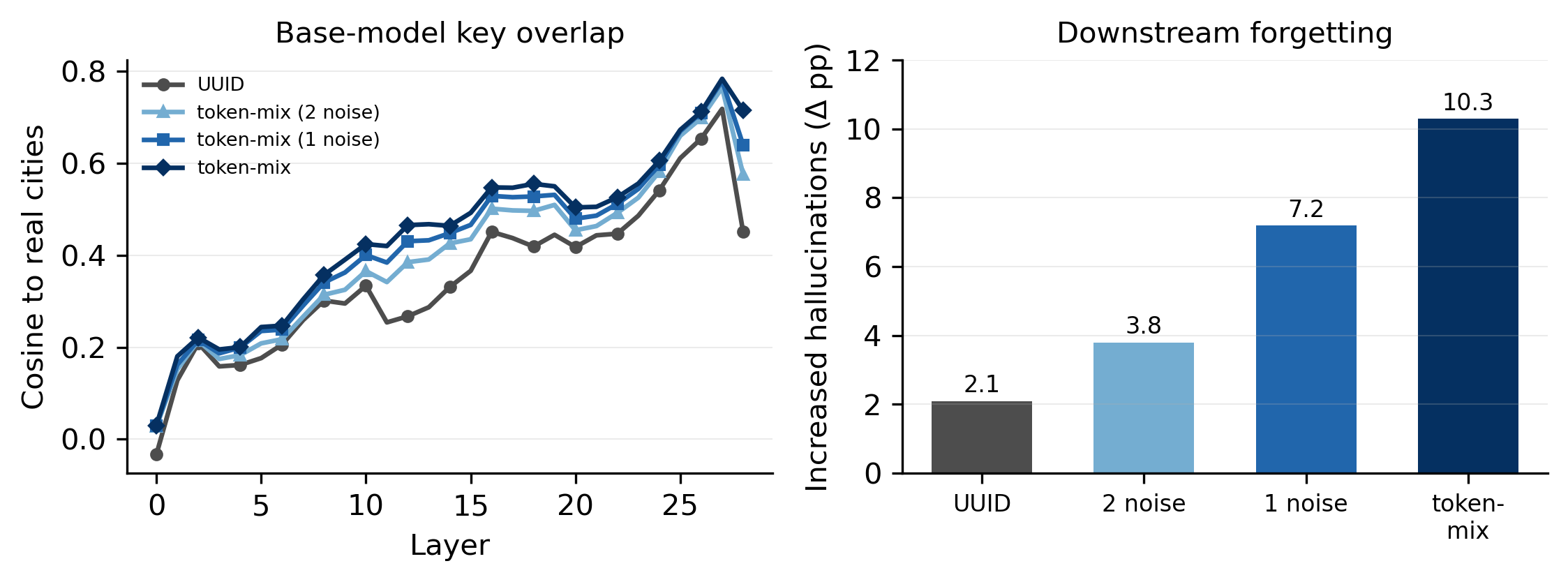}
\caption{\textbf{Dose-response: forgetting tracks base-model representational overlap.}
\textbf{Left:} per-layer cosine similarity of synthetic keys to real-city representations in the \emph{base} model, four graded-overlap conditions. \textbf{Right:} forgetting at $10^{4}$ injected facts increases monotonically with overlap.}
\label{fig:dose-layers}
\end{figure}

\paragraph{Cosine similarity to real cities.}
\Cref{fig:dose-layers} (left) reports the mean pairwise cosine similarity between entities of each synthetic condition and real cities. The ordering UUID $<$ 2-noise $<$ 1-noise $<$ token-mix holds at every layer past the earliest ones, with last-layer values $0.45$, $0.58$, $0.64$, $0.71$. Within-condition similarity to \emph{real--real} pairs behaves analogously.

\paragraph{Linear ``is a city'' probe.}
We train an $\ell_2$-regularized logistic probe on base-model representations to separate real cities from non-place entities, and evaluate the predicted $P(\text{city})$ on each synthetic condition. At layer 27, the deepest layer at which the probe is reliable, it transfers with the same ordering as the cosine analysis: real $0.91$, token-mix $0.71$, 1-noise $0.62$, 2-noise $0.50$, UUID $0.29$. The clean graded ordering holds in the final layers (25--28); at mid-network layers the probe itself is unreliable, with the UUID condition in particular oscillating widely. The cosine measure instead orders the conditions monotonically at every layer.

\paragraph{Takeaway.}
Both measures agree that the graded key conditions occupy graded positions in the base model's representation space before any fine-tuning, so the dose-response in forgetting (\cref{fig:dose-response}) tracks pre-existing representational overlap rather than an artifact of training. Forgetting values follow the convention of \cref{fig:scale-curves}, measuring the drop from the ${\approx}0.98$ plateau reached after task-format learning.

\section{Further Analyses on Self-Distillation}
\label{app:drift}

We report two sets of additional analyses that complement 
\secref{why_self-distillation}, both using the 10K semantic-overlap 
setting of \secref{semantic_overlap}: ablations on which part of the 
output distribution the distillation constraint needs to target, and an 
extended battery of drift metrics comparing SFT and self-distillation 
across representation and output space.

\paragraph{Which part of the output distribution matters?}
The main text establishes that $\ell_2$ weight regularization does not 
replicate the self-distillation benefit, suggesting the effect is not 
simply a consequence of smaller weight updates. We additionally ask 
whether the \emph{full} output distribution is necessary, or whether 
constraining a specific region suffices. This question connects to a 
broader literature on knowledge distillation: \citet{hinton2015distillingknowledgeneuralnetwork} 
showed that the ``dark knowledge'' carried by a teacher's soft 
probability distribution over non-target classes is more informative 
than the hard labels alone, precisely because it encodes the teacher's 
relative confidence over semantically related alternatives. In our 
setting, those alternatives are competing entity tokens --- and we 
hypothesize that it is constraining \emph{this} region, rather than 
the full distribution, that drives the forgetting reduction. We compare 
three variants of the distillation objective:

\textbf{Full self-distillation:} KL divergence over the complete 
vocabulary at each token position (\cref{eq:lwf_distill}).

\textbf{Top-$k$ distillation:} KL divergence restricted to the 
teacher's $k$ highest-probability tokens at each position. Let 
$T_k(j) = \arg\mathrm{top}_k\, p_{\theta_i}(\,\cdot\mid x_{<j})$ be 
the top-$k$ index set at position $j$, and let $p^{(k)}_{\theta_i,j}$, 
$p^{(k)}_{\theta,j}$ be the student and teacher distributions 
renormalized over $T_k(j)$. The distillation loss becomes
\begin{equation}
\mathcal{L}_{\mathrm{distill}}^{(k)}(\theta;\theta_i) = 
\mathbb{E}_{(x,y)\sim\mathcal{B}}\!\left[\frac{\tau^2}{|M(y)|}
\sum_{j\in M(y)}\mathrm{KL}\!\left(p^{(k)}_{\theta_i,j}\,\big\|\,
p^{(k)}_{\theta,j}\right)\right].
\end{equation}
We set $k = 0.05\%$ of the vocabulary (76 tokens), capturing between 
$91\%$ and $92\%$ of the teacher's probability mass per position.

\textbf{Random-$k$ distillation:} identical to top-$k$, but the 
76 tokens at each position are drawn uniformly at random rather than by 
probability rank. This condition controls for the \emph{number} of 
constrained logits while removing alignment with the teacher's 
high-probability region. Importantly, it does not reduce to standard 
SFT: the KL term remains active but constrains the student toward a 
uniform distribution over an arbitrary token subset, effectively 
injecting noise rather than meaningful structure.

Top-$k$ distillation fully replicates full self-distillation: forgetting 
on $\DHeld$ decreases to ${\approx}3$ pp while $\DUnk$ 
acquisition proceeds at the same pace. Random-$k$ distillation leaves 
forgetting at ${\approx}10$ pp, indistinguishable from 
standard SFT. What matters is therefore not the number of constrained 
logits but constraining the high-probability region specifically. This 
is consistent with \citet{hinton2015distillingknowledgeneuralnetwork}'s insight that soft 
targets carry informative relational structure --- here, the relative 
probabilities over competing entity tokens encode the teacher's 
``opinion'' about which alternatives are plausible, and preserving this 
opinion prevents the gradient updates for new entities from 
redistributing probability mass away from existing ones. 
\citet{buzzega2020darkexperiencegeneralcontinual, li2025bildbidirectionallogitsdifference} make a related observation in the continual 
learning setting, showing that dark experience replay --- which stores 
and replays soft teacher outputs rather than hard labels --- provides a 
stronger anti-forgetting signal than hard-label replay alone, precisely 
because the soft outputs encode richer relational structure. Our top-$k$
result sharpens this: what matters is not the full soft distribution,
but the high-probability region where semantically competing candidates reside.

\paragraph{Notation for drift metrics.}
Let $H^{(l)}_\theta(x) \in \mathbb{R}^d$ denote the hidden 
representation of the key entity's final token at transformer layer $l$ 
under model $\theta$, for input $x$. We write $H^{(l)}_0(x) = 
H^{(l)}_{\theta_0}(x)$ for the pretrained model and $H^{(l)}_i(x)$ for 
epoch $i$. All representations are extracted from the final token of the key entity and $\ell_2$-normalized before computing cosine similarities. The final token is the natural extraction point for entity representations in causal language models: unlike bidirectional architectures, where the full contextual representation of a phrase can reside at any of its token positions~\citep{kaplan2025followflowinformationflow}, causal attention ensures that only the last token has attended to all preceding tokens and thus carries the complete representation of the entity. We fix
$l = 14$: the middle layer of the 28-layer Qwen~2.5-1.5B model, where
the inner lexicon of LLMs is most reliably
encoded~\citep{kaplan2025tokenswordsinnerlexicon}. Let $\mathcal{P} = 
\{(x_u, x_h) : x_u \in \DUnk,\, x_h \in \DHeld,\, a(x_u) = a(x_h)\}$ 
be the set of unknown/held-out input pairs sharing the same ground-truth 
answer $a(\cdot)$, and $I = T_k^{\theta_0}(x_h)$ the pretrained model's 
top-$k$ token indices at the answer position for held-out input $x_h$.

The battery of metrics below is designed to triangulate the interference 
account from multiple angles: RD measures whether individual held-out 
representations move; NS measures whether their \emph{relative geometry} 
is preserved; MTD measures whether they drift specifically \emph{toward} 
new-entity representations; and Rank-$\rho$, NTR, and JSD measure 
whether this interference is visible at the output level. Convergent 
evidence across all six metrics strengthens the structural account over 
capacity-based or behavioral alternatives.

\paragraph{Hidden-state drift (RD).}
RD (defined in \secref{why_self-distillation}) measures the average cosine distance between held-out entity
representations before and after fine-tuning. Cosine distance captures 
directional shifts in representation space, which are more sensitive to 
semantic change than Euclidean distance~\citep{kim2025measuringrepresentationalshiftscontinual}. Under 
SFT, $\mathrm{RD}$ reaches ${\approx}11\%$; under self-distillation it 
stabilizes near $5\%$. The residual $5\%$ present in all conditions 
reflects task-format learning and is not associated with forgetting, as 
confirmed by the UUID condition in \secref{semantic_overlap}.

\paragraph{Neighborhood structure (NS).}
\begin{equation}
\mathrm{NS}_i = \frac{1}{|\mathcal{S}|^2}\sum_{x, x' \in \mathcal{S}}
\bigl|\cos\!\bigl(H^{(14)}_i(x), H^{(14)}_i(x')\bigr) - \cos\!\bigl(
H^{(14)}_0(x), H^{(14)}_0(x')\bigr)\bigr|,
\end{equation}
where $\mathcal{S} \subseteq \DHeld$ is a random subsample. NS measures 
how much the pairwise geometric structure among held-out entity 
representations changes during training --- capturing whether entities 
move \emph{together} (high RD, low NS) or \emph{rearrange relative to 
each other} (high NS). Under SFT, NS reaches ${\approx}0.07$; under 
self-distillation it stabilizes near $0.02$. The fact that both RD and 
NS are elevated under SFT indicates that held-out representations are 
not merely translating as a block, but genuinely reorganizing --- 
consistent with localized interference rather than a global drift of all 
representations. This pattern mirrors findings in \citet{masip2026puttingfaceforgettingcontinual}, 
who show that catastrophic forgetting in classification networks manifests
as reorganization concentrated among semantically similar classes, rather
than uniform drift.

\paragraph{Mean targeted drift (MTD).}
\begin{equation}
\mathrm{MTD}_i = \mathbb{E}_{(x_h, x_s)}\bigl[\cos\!\bigl(H^{(14)}_i
(x_h), H^{(14)}_i(x_s)\bigr) - \cos\!\bigl(H^{(14)}_0(x_h), 
H^{(14)}_0(x_s)\bigr)\bigr],
\end{equation}
where $x_h \in \DHeld$ and $x_s$ indexes the synthetic (newly trained) 
entities. MTD directly tests the directionality of the interference: 
does the representational drift push held-out entities \emph{toward} or 
\emph{away from} the new entities? Under SFT, MTD drops to 
${\approx}{-}0.09$: held-out and synthetic
representations diverge --- consistent with the model carving out 
distinct new-entity clusters that displace nearby held-out 
representations sideways rather than absorbing them. Under 
self-distillation, MTD remains near $0$, indicating that the 
cross-group representational relationship is largely preserved. This 
divergence pattern is consistent with \citet{nishi2025representationshatteringtransformerssynthetic}, 
who show that knowledge editing causes ``representation shattering'' --- 
a reorganization of the local neighborhood structure around edited 
entities --- and that nearby unedited entities are displaced as a 
byproduct.

\paragraph{Output-distribution drift (Rank-$\rho$).}
For each pair $(x_u, x_h) \in \mathcal{P}$, Rank-$\rho$ measures the 
Spearman rank correlation between the current and pretrained models'
logit rankings over unknown-fact and held-out inputs, restricted to the
shared index set $I$:
\begin{equation}
\mathrm{Rank}\text{-}\rho = \frac{1}{|\mathcal{P}|}\sum_{(x_u, x_h)
\in\mathcal{P}} \rho_{\mathrm{Spearman}}\!\bigl(z_\theta(x_u)_I,\; 
z_{\theta_0}(x_h)_I\bigr),
\label{eq:rank_rho}
\end{equation}
where $z_\theta(x)_I$ denotes the logits of $\theta$ at the answer 
position restricted to $I$. A rising Rank-$\rho$ indicates that the 
model's output for new-fact inputs increasingly ranks entity candidates 
in the same order as the pretrained model ranked them for held-out 
inputs --- a signature of the two entity neighborhoods converging in 
output space, which is precisely the interference pattern the structural 
account predicts. Under SFT, Rank-$\rho$ rises from ${\approx}0.32$ to 
${\approx}0.40$; under self-distillation it stays near $0.32$.

\begin{figure}[t]
\centering
\includegraphics[width=\textwidth]{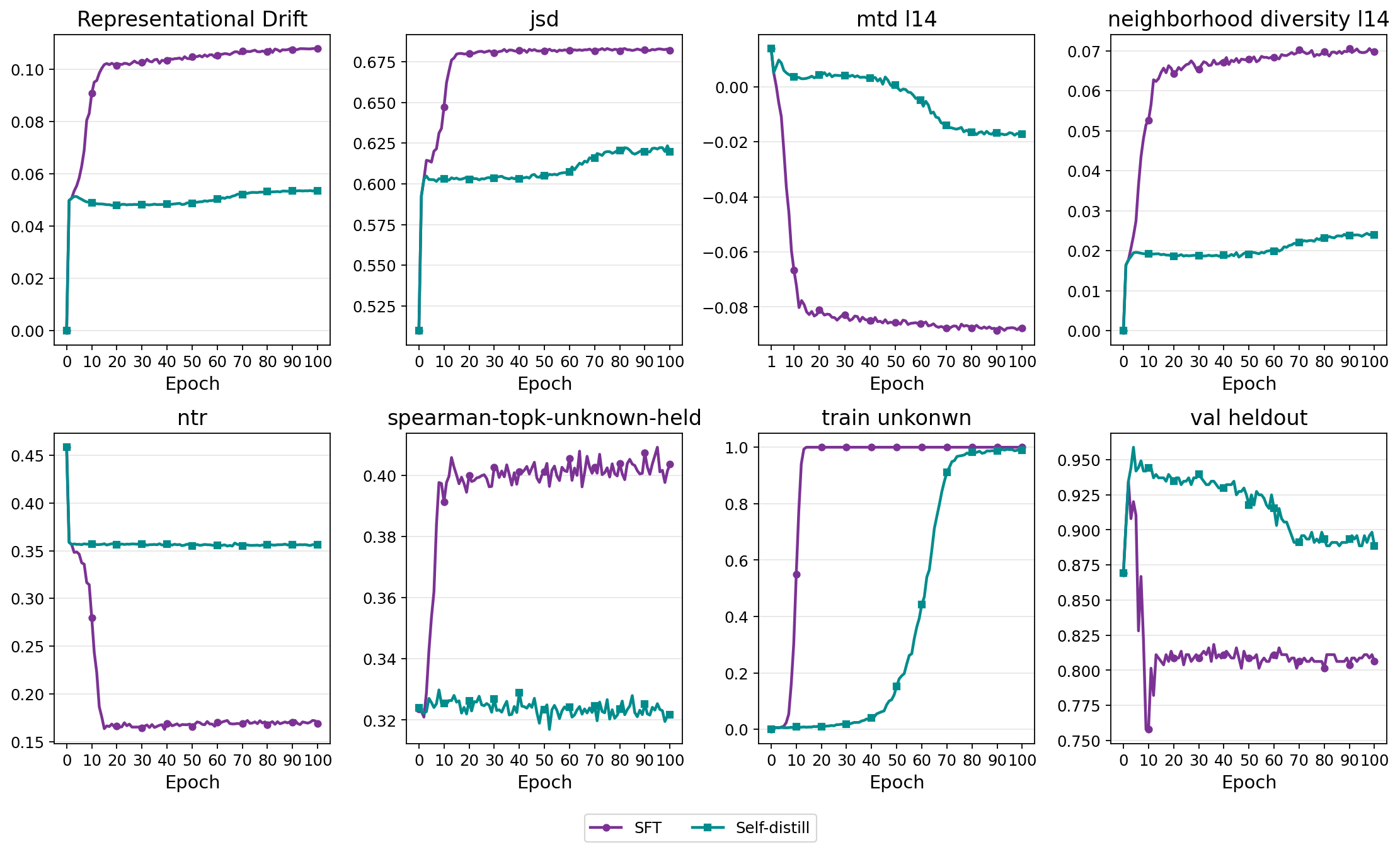}
\caption{
\textbf{Extended drift metrics under standard SFT and self-distillation} 
over 100 training epochs on the 10K semantic-overlap setting 
(\secref{semantic_overlap}). \textit{Top row, left to right:} 
Hidden-state drift (RD), Jensen-Shannon divergence (JSD), mean targeted 
drift (MTD), and neighborhood structure (NS) --- all at layer~14. 
\textit{Bottom row, left to right:} Neighborhood token ratio (NTR), 
output-distribution drift (Rank-$\rho$), training accuracy on $\DUnk$, 
and validation accuracy on $\DHeld$. All metrics consistently show less 
drift under \textcolor{teal}{self-distillation} than under 
\textcolor{violet}{standard SFT}. Factual acquisition on $\DUnk$ 
proceeds at a comparable pace across conditions.
}
\label{fig:drift_metrics}
\end{figure}

\paragraph{Neighborhood token ratio (NTR).}
\begin{equation}
\mathrm{NTR} = \mathbb{E}_{(x_u, x_h) \in \mathcal{P}}\!\left[
\frac{|T_k^\theta(x_u) \cap T_k^\theta(x_h)|}{k}\right],
\end{equation}
where $T_k^\theta(x)$ is the current model's top-$k$ token set at the 
answer position for input $x$. NTR complements Rank-$\rho$ by measuring 
overlap in the \emph{current} model's candidate sets (rather than 
correlation with the pretrained reference): do new-fact queries and 
held-out queries share the same answer candidates under the fine-tuned 
model? At initialization, the two groups share ${\approx}45\%$ of their 
top-$k$ candidates. Under SFT, NTR drops to ${\approx}17\%$ as the 
model learns specialized answer distributions for the new entities; 
under self-distillation it remains near $36\%$. The drop under SFT 
reflects the model pulling new-entity outputs into a distinct region of 
output space --- consistent with the MTD finding that held-out and 
synthetic representations diverge.

\paragraph{Jensen-Shannon divergence (JSD).}
\begin{equation}
\mathrm{JSD} = \mathbb{E}_{(x_u, x_h) \in \mathcal{P}}\!\left[
\mathrm{JSD}\!\left(p_\theta(x_u)_I \;\big\|\; p_{\theta_0}(x_h)_I
\right)\right],
\end{equation}
where $p_\theta(x)_I$ is the softmax distribution of $\theta$ at the 
answer position renormalized over $I$. JSD captures differences in the 
full probability shape rather than just rankings (complementing 
Rank-$\rho$). Both conditions show an upward trend, as the model 
naturally diverges from the fixed pretrained reference as it learns new 
facts. However, SFT reaches ${\approx}0.68$ versus ${\approx}0.63$ 
under self-distillation. The smaller gap here compared to Rank-$\rho$ 
and NTR reflects the fact that JSD is sensitive to all probability mass, 
including the task-format learning component shared across conditions; 
Rank-$\rho$ and NTR, by focusing on the entity-candidate region, isolate 
the interference-specific signal more cleanly.

\paragraph{Training and evaluation curves.}
The \textit{train unknown} and \textit{val heldout} subplots in 
\cref{fig:drift_metrics} confirm that both conditions achieve comparable 
factual acquisition on $\DUnk$ --- self-distillation converges somewhat 
later --- while held-out accuracy on $\DHeld$ is substantially better 
preserved under self-distillation, consistent with the main results of 
\secref{self-distillation}. Taken together, all six metrics tell a 
consistent story: SFT on semantically overlapping entities produces 
representational and output-space interference that is absent under
self-distillation, consistent with this interference being a primary
driver of factual forgetting.

\section{Derivations for the Hebbian Associative-Memory Account}
\label{app:theory}

This appendix collects the derivations referenced in \secref{am}. Each subsection is
self-contained and can be read independently.

\paragraph{Notation.} The main text indexes the held-out fact by its subject $x_h$, with
stored answer $a_h$ and margin $m_h$. The derivations below label the held-out \emph{fact}
by $A$, with subject $x_A$, answer $a_A$, margin $m_A$, and key $k_A:=k_{x_A}$; the
correspondence is $x_A=x_h$, $a_A=a_h$, $m_A=m_h$, $k_A=k_{x_h}$, $\rho_{jA}=\rho_{jh}$, and
$S_A=S_h$, and $\mathcal{A}$ denotes the answer set over which the free index $y$ ranges.
Eq.~\eqref{eq:factorize} gives the effect of a single injection $j$;
$\delta_y(A)=\sum_{j=1}^{M}\delta^{(j)}_y(x_h)$ below is the cumulative perturbation from
all $M$ injections, which \secref{am} describes in prose.

\subsection{Hebbian update derivation}
\label{app:hebbian-update}

\begin{lemma}[SFT update is a rank-one Hebbian increment]
\label{lem:hebbian}
Consider one gradient step of softmax cross-entropy on a new fact $(x_j,a_j)$ with learning
rate $\eta$, treating $W$ in Eq.~\eqref{eq:W} as a free linear readout and holding the key map
$x\mapsto k_x$ fixed. The induced update is
\begin{equation}
\label{eq:app-update}
\Delta W_j \;=\; \eta\,(u_{a_j}-\bar u_j)\,k_{x_j}^{\top},
\qquad
\bar u_j \;:=\; \sum_{y\in\mathcal{A}} p_y(x_j)\,u_y,
\qquad
p(x_j)=\mathrm{softmax}(s(x_j)).
\end{equation}
\end{lemma}

\begin{proof}
Let $\mathcal{L}_j$ denote the softmax-cross-entropy loss on $(x_j,a_j)$ with logits
$s_y=u_y^{\top}Wk_{x_j}$. The gradient with respect to logits is
\begin{equation*}
\frac{\partial\mathcal{L}_j}{\partial s_y} \;=\; p_y(x_j)-\indicator{y=a_j}.
\end{equation*}
By the chain rule,
\begin{equation*}
\nabla_W \mathcal{L}_j
\;=\;\sum_{y\in\mathcal{A}}\big(p_y(x_j)-\indicator{y=a_j}\big)\,u_y\, k_{x_j}^{\top}
\;=\;(\bar u_j-u_{a_j})\,k_{x_j}^{\top},
\end{equation*}
and $\Delta W_j=-\eta\nabla_W\mathcal{L}_j$ gives~\eqref{eq:app-update}.
\end{proof}

The identification is structural: the update~\eqref{eq:app-update} is exactly a rank-one
outer-product increment, the same primitive that builds the stored memory
$W=\sum_x u_{f^\star(x)}k_x^{\top}$ in Eq.~\eqref{eq:W}. SFT on $M$ new facts yields, to first
order, $W'=W+\sum_{j=1}^{M}\Delta W_j$. The two approximations are (i) first order in $\eta$
(per-example rank-ones rather than a fully entangled trajectory), and (ii) the key map is held
fixed during fine-tuning.

\subsection{Bias and variance of \texorpdfstring{$\delta_y(A)$}{delta\_y(A)}}
\label{app:bias-var}

We work under three conditions on the SFT corpus, the first two structural and the third a
mild regularity:
\begin{itemize}
\item[(a)] (Disjoint held-out answer.) $a_A\notin\{a_j\}_{j=1}^M$.
\item[(b)] (Distinct injected answers.) $a_j\neq a_{j'}$ for $j\neq j'$.
\item[(R)] (Cross-prior regularity.) There exists $C_R\ge 0$ such that for every $j\in[M]$ and
every $a\in\mathcal{A}$,
\begin{equation}
\label{eq:app-R}
\sum_{j'\neq j} p_a(x_{j'}) \;\le\; C_R.
\end{equation}
\end{itemize}
Conditions~(a) and~(b) are idealizations of our experimental corpora: in the synthetic settings, values are country names that repeat across injected facts, and held-out answers are drawn from the same set, so both conditions are violated in practice. We adopt them to isolate the overlap mechanism in its cleanest form. With repeated answers, the Gram matrix $B$ below acquires additional unit off-diagonal entries whose row sums grow with the per-answer repetition count; the bound then holds with a larger constant that depends on this count, and the $S_A$-monotone form of the variance is unchanged at any fixed repetition profile. Condition~(R) says that the total prior mass that any single answer accumulates across the other $M{-}1$ new
subjects is bounded; a sufficient condition is $\max_{j,\,a} p_a(x_j)\le C_R/M$, which is plausible in our regime since injected facts are \emph{Unknown} (the model assigns low probability even to their target answers before fine-tuning).

\begin{proposition}[Bias-variance decomposition]
\label{prop:bias-var}
Define $\delta_y(A):=u_y^{\top}(W'-W)k_A$, the perturbation of the readout score of answer $y$
on the held-out subject $x_A$, and write
$\delta_y(A)=\mu_y(A)+\xi_y(A)$ with $\mu_y(A):=\mathbb{E}_u[\delta_y(A)]$ and
$\xi_y(A):=\delta_y(A)-\mu_y(A)$, where the expectation is over random-spherical unembeddings
($u_a$ unit-norm with $\mathbb{E}[u_a u_a^{\top}]=I_d/d$).

\emph{(a) Bias (exact).} Under (a) and (b), for every $y\in\mathcal{A}$,
\begin{equation}
\label{eq:app-mu-y}
\mu_y(A) \;=\; \eta\sum_{j=1}^{M}\big(\indicator{y=a_j}-p_y(x_j)\big)\,\rho_{jA}.
\end{equation}
In particular, for $y=a_j$ (an injected answer),
\begin{equation}
\label{eq:app-mu}
\mu_{a_j}(A)
\;=\;\eta\,(1-p_{a_j}(x_j))\,\rho_{jA}
\;-\;\eta\sum_{j'\neq j} p_{a_j}(x_{j'})\,\rho_{j'A},
\end{equation}
and for $y\notin\{a_j\}_{j=1}^M$,
\begin{equation}
\label{eq:app-mu-other}
\mu_y(A) \;=\; -\eta\sum_{j=1}^{M} p_y(x_j)\,\rho_{jA}.
\end{equation}
The diagonal term in~\eqref{eq:app-mu} is the systematic ``pull'' of fact $j$ on the readout
for $A$, linear in $\rho_{jA}$, with prefactor $1-p_{a_j}(x_j)\in(0,1)$ that vanishes only as
SFT saturates $p_{a_j}(x_j)\to 1$. Under (R), the cross term in~\eqref{eq:app-mu} and the
bias~\eqref{eq:app-mu-other} are bounded in absolute value by $\eta\,(1+C_R)\max_j|\rho_{jA}|$.

\emph{(b) Variance (upper bound).} Under (a), (b), and (R), for every $y\in\mathcal{A}$,
\begin{equation}
\label{eq:app-var}
\mathrm{Var}_u[\xi_y(A)] \;\le\; \frac{C\,\eta^2}{d}\,S_A,
\qquad
S_A := \sum_{j=1}^{M}\rho_{jA}^2 \;=\; M\,\overline{\rho_A^2},
\qquad
C \;\le\; 2C_R+3.
\end{equation}
\end{proposition}

\begin{proof}
\textbf{Regrouping by answer token.}
Reindex the cumulative update $W'-W = \sum_{j=1}^{M}\Delta W_j$
(with $\Delta W_j$ from Eq.~\eqref{eq:app-update}) by answer rather than by fact:
\begin{equation}
\label{eq:app-Wregroup}
W'-W \;=\; \eta\sum_{a\in\mathcal{A}} u_a\, v_a^{\top},
\qquad
v_a \;:=\; \sum_{j=1}^{M} \alpha_{a,j}\, k_{x_j},
\qquad
\alpha_{a,j} \;:=\; \indicator{a=a_j}-p_a(x_j).
\end{equation}
Substituting,
\begin{equation}
\label{eq:app-delta-regroup}
\delta_y(A) \;=\; \eta\sum_{a\in\mathcal{A}} (u_y^{\top} u_a)\, (v_a^{\top} k_A).
\end{equation}

\textbf{Bias.}
Under the unembedding isotropy, $\mathbb{E}_u[u_y^{\top} u_a]=\indicator{y=a}$. Taking
expectations of~\eqref{eq:app-delta-regroup},
\begin{equation*}
\mu_y(A) \;=\; \eta\, v_y^{\top} k_A \;=\; \eta\sum_{j=1}^{M}\alpha_{y,j}\,\rho_{jA},
\end{equation*}
which is~\eqref{eq:app-mu-y}. The special cases~\eqref{eq:app-mu} and~\eqref{eq:app-mu-other}
follow by substituting $y=a_j$ (condition (b) collapses $\indicator{a_j=a_{j'}}$ to $j'=j$)
and $y\notin\{a_j\}$ (the indicator vanishes everywhere). For the bounds stated below the
proposition: by the $L^1/L^\infty$ inequality and condition (R),
\begin{equation*}
\Big|\sum_j p_y(x_j)\,\rho_{jA}\Big|
\;\le\;\Big(\sum_j p_y(x_j)\Big)\max_j |\rho_{jA}|
\;\le\; (1+C_R)\max_j |\rho_{jA}|,
\end{equation*}
using $\sum_j p_y(x_j) = p_y(x_{j_0})+\sum_{j\neq j_0} p_y(x_j)\le 1+C_R$ for any $j_0$.

\textbf{Variance: reduction to a Gram matrix.}
Subtracting the mean from~\eqref{eq:app-delta-regroup},
\begin{equation*}
\xi_y(A) \;=\; \eta\sum_{a\neq y}(u_y^{\top} u_a)\,(v_a^{\top} k_A)
\end{equation*}
(the $a=y$ term is exactly $\mu_y$). Under isotropy, $\{u_y^{\top} u_a\}_{a\neq y}$ are
pairwise uncorrelated with $\mathrm{Var}_u[u_y^{\top} u_a]=1/d$, hence
\begin{equation}
\label{eq:app-var-step1}
\mathrm{Var}_u[\xi_y(A)] \;=\; \frac{\eta^2}{d}\sum_{a\neq y}(v_a^{\top} k_A)^2
\;\le\; \frac{\eta^2}{d}\,\|V k_A\|^2,
\end{equation}
where $V$ is the operator with columns $\{v_a\}_{a\in\mathcal{A}}$. Expanding the squared norm,
\begin{equation}
\label{eq:app-VkA}
\|V k_A\|^2 \;=\; \sum_{a\in\mathcal{A}}(v_a^{\top} k_A)^2
\;=\; \sum_{j,j'=1}^{M}\rho_{jA}\,\rho_{j'A}\,B_{jj'}
\;=\; \rho_A^{\top} B\, \rho_A,
\end{equation}
where $\rho_A:=(\rho_{1A},\dots,\rho_{MA})^{\top}$ and $B\in\mathbb{R}^{M\times M}$ is the Gram
matrix
\begin{equation}
\label{eq:app-B}
B_{jj'} \;:=\; \sum_{a\in\mathcal{A}}\alpha_{a,j}\,\alpha_{a,j'}
\;=\; \indicator{a_j=a_{j'}}-p_{a_j}(x_{j'})-p_{a_{j'}}(x_j)+\langle p(x_j),p(x_{j'})\rangle.
\end{equation}
The last equality expands $\alpha_{a,j}\alpha_{a,j'}$ and sums over $a$ (the indicator product
collapses to $\indicator{a_j=a_{j'}}$ because both indicators are simultaneously $1$ only when
$a=a_j=a_{j'}$). Since $B$ is symmetric PSD (a Gram matrix),
$\rho_A^{\top} B\,\rho_A\le \|B\|_\mathrm{op}\,\|\rho_A\|_2^2 = \|B\|_\mathrm{op}\,S_A$.

\textbf{Bounding $\|B\|_\mathrm{op}$ via Gershgorin.}
On the diagonal, $\indicator{a_j=a_j}=1$ and $\|p(x_j)\|_2^2\le\max_a p_a(x_j)\cdot\sum_a p_a(x_j)\le 1$, so
\begin{equation*}
B_{jj} \;=\; 1-2 p_{a_j}(x_j)+\|p(x_j)\|_2^2 \;\le\; 2.
\end{equation*}
Off the diagonal, condition (b) gives $\indicator{a_j=a_{j'}}=0$ for $j\neq j'$, and the three
remaining terms are nonnegative (probabilities and inner products of probability vectors), so
\begin{equation*}
|B_{jj'}| \;\le\; p_{a_j}(x_{j'})+p_{a_{j'}}(x_j)+\langle p(x_j),p(x_{j'})\rangle.
\end{equation*}
Summing $|B_{jj'}|$ over $j'\neq j$ and applying (R) three times:
\begin{align*}
\sum_{j'\neq j} p_{a_j}(x_{j'}) &\;\le\; C_R \quad\text{(by (R) with $a=a_j$),}\\
\sum_{j'\neq j} p_{a_{j'}}(x_j) &\;\le\; \sum_{a\in\mathcal{A}} p_a(x_j) \;=\; 1
\quad\text{(by (b), the $a_{j'}$ are distinct in $\mathcal{A}$),}\\
\sum_{j'\neq j}\langle p(x_j),p(x_{j'})\rangle
&\;=\; \sum_{a\in\mathcal{A}} p_a(x_j)\sum_{j'\neq j} p_a(x_{j'})
\;\le\; C_R\sum_a p_a(x_j) \;=\; C_R \quad\text{(by (R)).}
\end{align*}
Hence $\sum_{j'\neq j}|B_{jj'}|\le 2C_R+1$. Since $B$ is symmetric PSD, Gershgorin's circle
theorem gives
\begin{equation*}
\|B\|_{\mathrm{op}} \;=\; \lambda_{\max}(B)
\;\le\; \max_{j}\Big(B_{jj}+\sum_{j'\neq j}|B_{jj'}|\Big)
\;\le\; 2+(2C_R+1) \;=\; 2C_R+3.
\end{equation*}
Combining~\eqref{eq:app-var-step1} and~\eqref{eq:app-VkA},
\begin{equation*}
\mathrm{Var}_u[\xi_y(A)]
\;\le\; \frac{\eta^2}{d}\,\|B\|_\mathrm{op}\,S_A
\;\le\; \frac{(2C_R+3)\,\eta^2}{d}\,S_A,
\end{equation*}
which is~\eqref{eq:app-var} with $C\le 2C_R+3$.
\end{proof}

\subsection{Forgetting probability and scaling law (heuristic)}
\label{app:forget-law}

Let $m_A := s_{a_A}(x_A) - \max_{y\neq a_A} s_y(x_A) > 0$ be the pre-finetuning margin of
fact $A$. Fact $A$ is forgotten when, for some competing answer $y\neq a_A$,
\begin{equation}
\label{eq:app-forget-cond}
\delta_y(A) - \delta_{a_A}(A) \;>\; m_A.
\end{equation}
By Proposition~\ref{prop:bias-var}(a), $\mu_y(A)$ is nonzero (to leading order in (R)) only
for $y\in\{a_j\}_{j=1}^M$; the maximum of $\mu_y(A)$ over $y\neq a_A$ is therefore attained at
some $y=a_j$. Writing $\delta = \mu + \xi$, the forgetting condition~\eqref{eq:app-forget-cond}
specializes to
\begin{equation}
\label{eq:app-forget-cond-aj}
\mu_{a_j}(A) - \mu_{a_A}(A) + \big(\xi_{a_j}(A) - \xi_{a_A}(A)\big) \;>\; m_A.
\end{equation}

\emph{Variance of the difference (rigorous).} By Proposition~\ref{prop:bias-var}(b) applied to
$y=a_j$ and $y=a_A$ separately, and the elementary inequality
$\mathrm{Var}[X-Y] \le 2\mathrm{Var}[X] + 2\mathrm{Var}[Y]$,
\begin{equation}
\label{eq:app-var-diff}
\mathrm{Var}_u\big[\xi_{a_j}(A) - \xi_{a_A}(A)\big]
\;\le\; \frac{4(2C_R+3)\,\eta^2}{d}\,S_A.
\end{equation}

\emph{Gaussian heuristic.} By the proof of Proposition~\ref{prop:bias-var},
$\xi_y(A) = \eta \sum_{a \neq y} (u_y^{\top} u_a)(v_a^{\top} k_A)$ is a linear combination of
$|\mathcal{A}|-1$ pairwise-uncorrelated mean-zero terms (the cross-products $u_y^{\top} u_a$
under isotropy), each of order $1/\sqrt d$. For large $|\mathcal{A}|$, a CLT approximation
gives that $\xi_{a_j}(A) - \xi_{a_A}(A)$ is approximately Gaussian. Combining
with~\eqref{eq:app-forget-cond-aj} and~\eqref{eq:app-var-diff}, the forgetting probability is
then a Gaussian tail:
\begin{equation}
\label{eq:app-forgetprob}
\Pr[\text{forget }A]
\;\approx\;
g\!\left(\frac{m_A - \max_j\big(\mu_{a_j}(A) - \mu_{a_A}(A)\big)}{\eta\sqrt{4(2C_R+3)\,S_A/d}}\right),
\end{equation}
with $g$ a decreasing function. This is the one heuristic step in the analysis: the Gaussian
approximation is what justifies a specific functional form. We
use~\eqref{eq:app-forgetprob} only to derive monotonic orderings (both the numerator decreases
and the denominator increases as $S_A$ grows), never to read off a specific exponent.
The compact form in the main text, \cref{eq:forget-law}, suppresses the bias term in the
numerator (valid in the small-$\mu$ regime) and absorbs the constant $4(2C_R+3)$ into $g$;
the full form~\eqref{eq:app-forgetprob} is monotone in $S_A$ either way.

\subsection{Why output-space regularization beats weight-space regularization}
\label{app:output-vs-weight}

The interference identified in \secref{am} lives along the unembedding directions: by Proposition~\ref{prop:bias-var}, both the bias $\mu_y(A)$ and the fluctuation $\xi_y(A)$ are perturbations of output logits. A KL penalty toward the pre-fine-tuning model therefore suppresses both directly, including on held-out subjects, which share the training subjects' representational geometry. A weight-space penalty instead bounds $\|\Delta W_j\|_F=\eta\,\|u_{a_j}-\bar u_j\|\,\|k_{x_j}\|$, a quantity independent of the held-out key $k_A$ and hence blind to the overlap $\rho_{jA}$: it protects all stored facts uniformly rather than preferentially protecting the overlap-vulnerable ones, consistent with the $\ell_2$ comparison of \secref{why_self-distillation} and the parameter-space baselines of \cref{tab:baselines}. The same logic explains why top-$k$ distillation succeeds while random-$k$ does not (\appref{app:drift}): by Eq.~\eqref{eq:app-mu-y}, the interference bias is nonzero, to leading order under condition~(R), only for the injected answers $y\in\{a_j\}_{j=1}^{M}$, exactly the high-probability directions that top-$k$ constrains, while random-$k$ spends most of its budget on directions carrying no bias to cancel.

\subsection{Extension to distributed storage}
\label{app:distributed}

The main text models factual storage as a single effective associative memory $W$ acting on
the subject's internal representation $k_x$ (Eq.~\eqref{eq:W}) without committing to a
specific layer at which $k_x$ is read. In real transformers, factual storage is distributed
across multiple attention and MLP components, each potentially reading the residual stream at
a different layer \citep{meng2023locatingeditingfactualassociations,meng2023masseditingmemorytransformer}. This appendix shows that the main-text
analysis extends to this distributed setting with no qualitative change.

Let component $c$ read at layer $\ell_c$ with key $k_x^{(c)}:=\phi_{\ell_c}(x)$, and let the
readout decompose as a weighted sum $s_y(x)=\sum_c w_c\,u_y^{\top}W^{(c)}k_x^{(c)}$ for
nonnegative component weights $w_c\ge 0$. Lemma~\ref{lem:hebbian} applies component-wise:
$\Delta W_j^{(c)}=\eta(u_{a_j}-\bar u_j)(k_{x_j}^{(c)})^{\top}$, so the perturbation is
\begin{equation}
\label{eq:app-distributed-delta}
\delta_y(A) \;=\; \sum_c w_c\,\eta\sum_{j=1}^M
\big(u_y^{\top}(u_{a_j}-\bar u_j)\big)\,\rho_{jA}^{(c)},
\qquad
\rho_{jA}^{(c)} := \langle k_{x_j}^{(c)},k_A^{(c)}\rangle.
\end{equation}
The bias-variance decomposition of Proposition~\ref{prop:bias-var} applies term by term.
Treating the components as independent in $u$ (a simplification; cross-component covariances
share the same $u_a$'s and are bounded by Cauchy-Schwarz), the variance bound becomes
\begin{equation}
\label{eq:app-distributed-var}
\mathrm{Var}_u[\xi_y(A)] \;\le\; \frac{C\,\eta^2}{d}\sum_c w_c^2 \sum_{j=1}^M \big(\rho_{jA}^{(c)}\big)^2.
\end{equation}
Without the independence simplification the worst-case bound takes the form
$(C\eta^2/d)(\sum_c w_c\sqrt{S_A^{(c)}})^2$, which is still monotone in each component overlap
$S_A^{(c)}:=\sum_j(\rho_{jA}^{(c)})^2$. The functional form of monotonicity in summed squared
overlap is preserved either way, with overlaps now measured at the layer of each component.
We do not pin down the weights $w_c$ or the per-component read layers, which depend on which
component dominates retrieval; what is preserved is the qualitative law.

\end{document}